%% file: ok.tex
\documentclass{article}

\PassOptionsToPackage{numbers}{natbib} 

 \usepackage[final]{neurips_2019}
\usepackage[utf8]{inputenc} 
\usepackage[T1]{fontenc}    
\usepackage{hyperref}       
\usepackage{url}            
\usepackage{booktabs}       
\usepackage{amsfonts}       
\usepackage{nicefrac}       
\usepackage{microtype}      
\usepackage[noend]{algcompatible}
\usepackage{algorithm}
\usepackage{subfigure}
\usepackage{amsmath}
\usepackage{amsfonts}
\usepackage{amssymb}
\usepackage{amsthm}
\usepackage{mathtools}
\usepackage{graphicx}
\usepackage{wrapfig}
\usepackage{etoolbox}

\title{The Option Keyboard \\ Combining Skills in Reinforcement Learning}

\author{%
  {\bf Andr\'e Barreto}, 
  {\bf Diana Borsa}, 
  {\bf Shaobo Hou}, 
  {\bf Gheorghe Comanici}, 
  {\bf Eser Ayg\"un},    \\
  {\bf Philippe Hamel}, 
  {\bf Daniel Toyama}, 
  {\bf Jonathan Hunt}, 
  {\bf Shibl Mourad}, 
  {\bf David Silver}, 
  {\bf Doina Precup}  \vspace{1.0mm}  \\
   \texttt{\small \{andrebarreto,borsa,shaobohou,gcomanici,eser\}@google.com} \\
   \texttt{\small \{hamelphi,kenjitoyama,jjhunt,shibl,davidsilver,doinap\}@google.com}  \vspace{1.5mm} \\
  DeepMind }

\include{defs}

\algrenewcommand\algorithmicloop{\textbf{repeat forever}}

\graphicspath{{fig/}}

\begin{document}

\maketitle

\begin{abstract}
The ability to combine known skills to create new ones may be crucial in the solution of complex reinforcement learning problems that unfold over extended periods. We argue that a robust way of combining skills is to define and manipulate them in the space of pseudo-rewards (or ``cumulants''). Based on this premise, we propose a framework for combining skills using the formalism of options. We show that every deterministic option can be unambiguously represented as a cumulant defined in an extended domain. Building on this insight and on previous results on transfer learning, we show how to approximate options whose cumulants are linear combinations of the cumulants of known options. This means that, once we have learned options associated with a set of cumulants, we can instantaneously synthesise options induced by any linear combination of them, without any learning involved. We describe how this framework provides a hierarchical interface to the environment whose abstract actions correspond to combinations of basic skills. We demonstrate the practical benefits of our approach in a resource management problem and a navigation task involving a quadrupedal simulated robot. 
\end{abstract}

\section{Introduction}
\label{sec:introduction}

In reinforcement learning (RL) an agent takes actions in an environment in order to maximise the amount of reward received in the long run~\cp{sutton2018reinforcement}. This textbook definition of RL treats actions as atomic decisions made by the agent at every time step. Recently, \ct{sutton2016towards} proposed a new view on action selection. In order to illustrate the potential benefits of his proposal \ca{sutton2016towards} resorts to the following analogy. Imagine that the interface between agent and environment is a \emph{piano keyboard}, with each key corresponding to a possible action. Conventionally the agent plays one key at a time and each note lasts exactly one unit of time. If we expect our agents to do something akin to playing music, we must generalise this interface in two ways. First, we ought to allow notes to be arbitrarily long---that is, we must replace actions with \emph{skills}. Second, we should be able to also play \emph{chords}. 

The argument in favour of temporally-extended courses of actions has repeatedly been made in the literature: in fact, the notion that agents should be able to reason at multiple temporal scales is one of the pillars of hierarchical RL~\cp{dayan92feudal,parr97reinforcement,sutton99between,dietterich2000hierarchical,kaelbling2014hierarchical}. The insight that the agent should have the ability to \emph{combine} the resulting skills is a far less explored idea. This is the focus of the current work. 

The possibility of combining skills replaces a monolithic action set with a combinatorial counterpart: by learning a small set of basic skills (``keys'') the agent should be able to perform a potentially very large number of combined skills (``chords''). For example, an agent that can both walk and grasp an object should be able to walk while grasping an object without having to learn a new skill. According to \ct{sutton2016towards}, this combinatorial action selection process ``could be the key to generating behaviour with a good mix of preplanned coherence and sensitivity to the current situation.'' 

But how exactly should one combine skills? One possibility is to combine them in the space of policies: for example, if we look at policies as distribution over actions, a combination of skills can be defined as a mixture of the corresponding distributions. One can also combine parametric policies by manipulating the corresponding parameters. Although these are feasible solutions, they fail to capture possible \emph{intentions} behind the skills. Suppose the agent is able to perform two skills that can be associated with the same objective---distinct ways of grasping an object, say. It is not difficult to see how combinations of the corresponding behaviours can completely fail to accomplish the common goal. We argue that a more robust way of combining skills is to do so directly in the goal space, using pseudo-rewards or \emph{cumulants}~\cp{sutton2018reinforcement}. If we associate each skill with a cumulant, we can combine the former by manipulating the latter. This allows us to go beyond the direct prescription of behaviours, working instead in the space of intentions.  

Combining skills in the space of cumulants poses two challenges. First, we must establish a well-defined mapping between cumulants  and skills. Second, once a combined cumulant is defined, we must be able to perform the associated skill without having to go through the slow process of learning it. We propose to tackle the former by adopting \emph{options} as our formalism to define skills~\cp{sutton99between}. We show that there is a large subset of the space of options, composed of deterministic options, in which every element can be unambiguously represented as a cumulant defined in an extended domain. Building on this insight, we extend \ca{barreto2017successor}'s \cp{barreto2017successor,barreto2018transfer} previous results on transfer learning to show how to approximate options whose cumulants are linear combinations of the cumulants of known options. This means that, once the agent has learned options associated with a collection of cumulants, it can instantaneously synthesise options induced by \emph{any} linear combination of them, \emph{without any learning involved}. Thus, by learning a small set of options, the agent instantaneously has at its disposal a potentially enormous number of combined options. Since we are combining cumulants, and not policies, the resulting options will be truly novel, meaning that they cannot, in general, be directly implemented as a simple alternation of their constituents.

We describe how our framework provides a flexible interface with the environment whose abstract actions correspond to combinations of basic skills. As a reference to the motivating analogy described above, we call this interface the \emph{option keyboard}. We discuss the merits of the option keyboard at the conceptual level and demonstrate its practical benefits in two experiments: a resource management problem and a realistic navigation task involving a quadrupedal robot simulated in MuJoCo~\cp{todorov2012mujoco,schulman2016high}.

\section{Background}
\label{sec:background}

As usual, we assume the interaction between agent and environment can be modelled as a \emph{Markov decision process} (MDP)~\cp{puterman94markov}. An MDP is a tuple $M \defi (\S,\A,p,r,\gamma)$, where $\S$ and $\A$ are the state and action spaces, $p(\cdot|s,a)$ gives the next-state distribution upon taking action $a$ in $s$, $r: \S \times \A \times \S \mapsto \R$ specifies the reward associated with the transition $s \xrightarrow{a} s'$, and $\gamma \in [0,1)$ is the discount factor. 

The objective of the agent is to find a \emph{policy} $\pi: \S \mapsto \A$ that maximises the expected \emph{return}
$
G_{t} \defi \sum_{i=0}^{\infty} \gamma^{i} R_{t+i},
$
where $R_t = r(S_t, A_t, S_{t+1})$.
A principled way to address this problem is to use methods derived from dynamic programming, which usually compute the \emph{action-value function} of a policy $\pi$ as:
$
\label{eq:Q}
Q^{\pi}(s,a) \defi \E^{\pi} \left[ G_{t} | S_t = s, A_t = a\right],
$
where $\E^{\pi}[\cdot]$ denotes expectation over the transitions induced by $\pi$~\cp{puterman94markov}. The computation of $Q^{\pi}(s,a)$ is called {\em policy evaluation}. Once $\pi$ has been evaluated, we can compute a greedy policy 
\begin{equation}
\label{eq:policy_improvement}
\pi'(s) \in \argmax_{a} Q^{\pi}(s,a) \; \text{ for all } s \in \S. 
\end{equation}
It can be shown that $Q^{\pi'}(s,a) \ge Q^{\pi}(s,a)$ for all $(s,a) \in \S \times \A$, and hence the computation of $\pi'$ is referred to as \emph{policy improvement}. The alternation between policy evaluation and policy improvement is at the core of many RL algorithms, which usually carry out these steps approximately. Here we will use a tilde over a symbol to indicate that the associated quantity is an approximation ({\sl e.g.}, $\Qt^\pi \approx Q^\pi$).

\subsection{Generalising policy evaluation and policy improvement}
\label{sec:gpe_gpi}

Following~\ct{sutton2018reinforcement}, we call any signal defined as $c: \S \times \A \times \S \mapsto \R$ a \emph{cumulant}. Analogously to the conventional value function $Q^\pi$, we define $Q^\pi_c$ as the expected discounted sum of cumulant $c$ under policy $\pi$~\cp{sutton2011horde}. Given a policy $\pi$ and a set of cumulants \C, we call the evaluation of $\pi$ under all $c \in \C$ \emph{generalised policy evaluation} (GPE)~\cp{barreto2020fast}. \ct{barreto2017successor,barreto2018transfer} propose an efficient form of GPE based on \emph{successor features}: they show that, given cumulants $c_1, c_2, ..., c_d$, for any $\cm = \sum_i w_i \cm_i$, with $\w \in \R^ d$,
{ 
\begin{equation}
 \label{eq:gpe}
Q^{\pi}_{\cm}(s,a) 
\defi \E^{\pi} \left[ \sum_{k=0}^{\infty} \gamma^k \sum_{i=1}^{d} w_i \Cm_{i, t + k} | S_t = s, A_t = a \right] \\
= \sum_{i=1}^{d} w_i Q^{\pi}_{\cm_i}(s,a),
\end{equation}
} \hspace{-2mm}
where $\Cm_{i, t} \defi c_i(S_t, A_t, R_t)$. Thus, once we have computed $Q^\pi_{\cm_1}, Q^\pi_{\cm_1}, ..., Q^\pi_{\cm_d}$, we can instantaneously evaluate $\pi$ under any cumulant in the set $\C \defi \{c = \sum_i w_i c_i \,|\, \w \in \R^d\}$.

Policy improvement can also be generalised. In \ctp{barreto2017successor} \emph{generalised policy improvement} (GPI) the improved policy is computed based on a set of value functions. Let $Q^{\pi_1}_{\cm}, Q^{\pi_2}_{\cm}, ... Q^{\pi_n}_{\cm}$ be the action-value functions of $n$ policies $\pi_i$ under cumulant $\cm$, and let $Q^{\max}_{\cm}(s,a) = \max_i Q^{\pi_i}_{\cm}(s,a)$ for all $(s,a) \in \S \times \A$. If we define 
\begin{equation}
\label{eq:gpi}
\pi(s) \in \argmax_a Q^{\max}_{\cm}(s,a) \text{ for all } s \in \S,
\end{equation}
then $Q^{\pi}_{\cm}(s,a) \ge Q^{\max}_{\cm}(s,a)$ for all $(s,a) \in \S \times \A$. This is a strict generalisation of standard policy improvement~(\ref{eq:policy_improvement}). The guarantee extends to the case in which GPI uses approximations $\Qt^{\pi_i}_{\cm}$~\cp{barreto2017successor}.

\subsection{Temporal abstraction via options}
\label{sec:options}

As discussed in the introduction, one way to get temporal abstraction is through the concept of \emph{options}~\cp{sutton99between}. Options are temporally-extended courses of actions. In their more general formulation, options can depend on the entire \emph{history} between the time $t$ when they were initiated and the current time step $t+k$, $h_{t:t+k} \defi s_t a_t s_{t+1}... a_{t+k-1} s_{t+k} $. Let \H\ be the space of all possible histories; a \emph{semi-Markov option} is a tuple $o \defi (\I_o, \pi_o, \beta_o)$ where $\I_o \subset \S$ is the set of states where the option can be initiated, $\pi_o: \H \mapsto \A$ is a policy over histories, and $\beta_o: \H \mapsto [0,1]$ gives the probability that the option  terminates after history $h$ has been observed~\cp{sutton99between}. It is worth emphasising that semi-Markov options depend on the history since their initiation, but not before.

\section{Combining options}
\label{sec:combining}
\vspace{-2mm}

In the previous section we discussed how several key concepts in RL can be generalised: rewards with cumulants, policy evaluation with GPE, policy improvement with GPI, and actions with options. In this section we discuss how these concepts can be used to combine skills.

\subsection{The relation between options and cumulants}
\label{sec:options_cumulants}

We start by showing that there is a subset of the space of options in which every option can be unequivocally represented as a cumulant defined in an extended domain.  

First we look at the relation between policies and cumulants. Given an MDP $(\S,\A, p, \cdot, \gamma)$, we say that a cumulant $\cm_{\pi}: \S \times \A \times \S \mapsto \R$ \emph{induces} a policy $\pi: \S \mapsto \A$ if $\pi$ is optimal for the MDP $(\S,\A, p, \cm_{\pi}, \gamma)$. We can always define a cumulant $\cm_{\pi}$ that induces a given policy $\pi$. For instance, if we make 
\begin{equation}
\label{eq:cumulant_pi}
\cm_{\pi}(s, a, \cdot) = \left\{\begin{array}{l}
                      0 \text{ if } a = \pi(s); \\
                      z \text { otherwise, }
                     \end{array}\right.
\end{equation}
where $z < 0$, it is clear that $\pi$ is the only policy that achieves the maximum possible value $Q^{\pi}(s,a) = Q^{*}(s,a) = 0$ on all $(s,a) \in \S \times \A$. In general, the relation between policies and cumulants is a many-to-many mapping. First, there is more than one cumulant that induces the same policy: for example, any $z < 0$ in~(\ref{eq:cumulant_pi}) will clearly lead to the same policy $\pi$. There is thus an infinite set of cumulants $\C_{\pi}$ associated with $\pi$. Conversely, although this is not the case in~(\ref{eq:cumulant_pi}), the same cumulant can give rise to multiple policies if more than one action achieves the maximum in (\ref{eq:policy_improvement}). 

Given the above, we can use any cumulant $\cm_{\pi} \in \C_{\pi}$ to refer to policy $\pi$. In order to extend this possibility to options $o = (\I_o, \pi_o, \beta_o)$ we need two things. First, we must define cumulants in the space of histories \H. This will allow us to induce semi-Markov policies $\pi_o: \H \mapsto \A$ in a way that is analogous to~(\ref{eq:cumulant_pi}). Second, we need cumulants that also induce the initiation set $\I_o$ and the termination function $\beta_o$. We propose to accomplish this by augmenting the action space. 

Let $\tau$ be a \emph{termination action} that terminates option $o$ much like the termination function $\beta_o$. We can think of $\tau$ as a fictitious action and model it by defining an augmented action space $\A^{+} \defi \A \cup \{\tau\}$. When the agent is executing an option $o$, selecting action $\tau$ immediately terminates it. We now show that if we extend the definition of cumulants to also include $\tau$ we can have the resulting cumulant induce not only the option's policy but also its initiation set and termination function. Let $\ecm: \H \times \A^{+} \times \S \mapsto \R$ be an \emph{extended cumulant}. Since \ecm\ is defined over the augmented action space, for each $h \in \H$ we now have a \emph{termination bonus} $\ecm(h, \tau, s) = \ecm(h, \tau)$ that determines the value of interrupting option $o$ after having observed $h$. The extended cumulant $\ecm$ induces an \emph{augmented policy} $\api_{\ecm}: \H \mapsto \A^{+}$ in the same sense that a standard cumulant induces a policy (that is, $\api_{\ecm}$ is an optimal policy for the derived MDP whose state space is \H\ and the action space is $\A^{+}$; see Appendix~\ref{sec:theory} for details). We argue that $\api_{\ecm}$ is equivalent to an option $o_{\ecm} \equiv (\I_{\ecm},\pi_{\ecm}, \beta_{\ecm})$  whose components are defined as follows.  The policy $\pi_{\ecm}: \H \mapsto \A$ coincides with $\api_{\ecm}$ whenever the latter selects an action in \A. The termination function is given by
\begin{equation}
\label{eq:termination}
\beta_{\ecm}(h) = \left\{ \begin{array}{l}
			  1  \; \text{ if } \ecm(h,\tau) > \max_{a \ne \tau} Q^{\api_{\ecm}}_{\ecm}(h,a), \\
			  0  \; \text{ otherwise.} 
                          \end{array}\right.
\end{equation}
In words, the agent will terminate after $h$ if the instantaneous termination bonus $\ecm(h,\tau)$ is larger than the maximum expected discounted sum of cumulant $\ecm$ under policy $\api_{\ecm}$. Note that when $h$ is a single state $s$, no concrete action has been executed by the option yet, hence it terminates with $\tau$ immediately after its initiation. This is precisely the definition of the initialisation set $\I_{\ecm} \defi \{s \,|\, \beta_{\ecm}(s) = 0\}$.

Termination functions like~(\ref{eq:termination}) are always deterministic. This means that 
extended cumulants $\ecm$ can only represent options $o_{\ecm}$ in which $\beta_{\ecm}$ is a mapping $\H \mapsto \{0,1\}$. In fact, it is possible to show that all options of this type, which we will call \emph{deterministic options}, are representable as an extended cumulant $\ecm$, as formalised in the following proposition (proof in Appendix~\ref{sec:theory}):

\begin{proposition}
\label{teo:det_opt_cumulants}
Every extended cumulant induces at least one deterministic option, and every deterministic option can be unambiguously induced by an infinite number of extended cumulants.
\end{proposition}

\subsection{Synthesising options using GPE and GPI}
\label{sec:combining_options}

In the previous section we looked at the relation between extended cumulants and deterministic options; we now build on this connection to use GPE and GPI to combine options.

Let $\Ec \defi \{\ecm_1, \ecm_2, ..., \ecm_d\}$ be a set of extended cumulants. We know that $\ecm_{i}: \H \times \A^{+} \times \S \mapsto \R$ is associated with deterministic option $o_{\ecm_i} \defi \api_{\ecm_i}$. As with any other cumulant, the extended cumulants $\ecm_{i}$ can be linearly combined; it then follows that, for any $\w \in \R^d$, $\ecm = \sum_i w_i \ecm_i$ defines a new deterministic option $o_{\ecm} \defi \api_{\ecm}$. Interestingly, the termination function of $o_{\ecm}$ has the form~(\ref{eq:termination}) with termination bonuses defined as
$
\label{eq:linear_comb}
\ecm(h,\tau) = \sum_i w_i \ecm_i(h,\tau) .
$
This means that the combined option $o_\ecm$ ``inherits'' its termination function from its constituents $o_{\ecm_i}$. Since any $\w \in \R^d$ defines an option $o_\ecm$, the set \Ec\ can give rise to a very large number of combined options.

The problem is of course that for each $\w \in \R^d$ we have to actually compute the resulting option $\api_{\ecm}$. This is where GPE and GPI come to the rescue. 
Suppose we have the values of options $\api_{\ecm_i}$ under all the cumulants $\ecm_1, \ecm_2, ..., \ecm_d$. With this information, and analogously to~(\ref{eq:gpe}), we can use the fast form of GPE provided by successor features to compute the value of $\api_{\ecm_j}$ with respect to $\ecm$:
 \begin{equation}
\label{eq:gpe_ecm}
Q^{\api_{\ecm_j}}_{\ecm}(h,a) = \sum_i w_i Q^{\api_{\ecm_j}}_{\ecm_i}(h,a) .
\end{equation}
Now that we have all the options $\api_{\ecm_j}$ evaluated under $\ecm$, we can merge them to generate a new option that does at least as well as, and in general better than, all of them. This is done by applying GPI over the value functions $Q^{\api_{\ecm_j}}_{\ecm}$:
\begin{equation}
\label{eq:gpi_ecm}
\tilde{\api}_\ecm(h) \in \argmax_{a \in \A^+} \max\nolimits_j Q^{\api_{\ecm_j}}_{\ecm} (h,a) .
\end{equation}
From previous theoretical results we know that 
$
\label{eq:lower_bound_gpi}
\max_{j} Q^{\api_{\ecm_j}}_{\ecm}(h,a) \le Q^{\tilde{\api}_{\ecm}}_{\ecm}(h,a) \le Q^{\api_{\ecm}}_{\ecm}(h,a)
$
for all $(h, a) \in \H \times \A^+$~\cp{barreto2017successor}. In words, this means that, even though the GPI option $\tilde{\api}_{\ecm}$ is not necessarily optimal, following it will in general result in a higher return in terms of cumulant \ecm\ than if the agent were to execute any of the known options $\api_{\ecm_j}$. Thus, we can use $\tilde{\api}_\ecm$ as an approximation to $\api_{\ecm}$ \emph{that requires no additional learning}. It is worth mentioning that the action selected by the combined option in~(\ref{eq:gpi_ecm}), $\tilde{\api}_{\ecm}(h)$, can be different from $\api_{\ecm_i}(h)$ for all $i$---that is, the resulting policy cannot, in general, be implemented as an alternation of its constituents. This highlights the fact that combining cumulants is not the same as defining a higher-level policy over the associated options. 

In summary, given a set of cumulants $\Ec$, we can combine them by picking weights $\w$ and computing the resulting cumulant $\ecm = \sum_i w_i \ecm_i$. This can be interpreted as determining how desirable or undesirable each cumulant is. Going back to the example in the introduction, suppose that $\ecm_1$ is associated with walking and $\ecm_2$ is associated with grasping an object. Then, cumulant $\ecm_1 + \ecm_2$ will reinforce both behaviours, and will be particularly rewarding when they are executed together. In contrast, cumulant $\ecm_1 - \ecm_2$ will induce an option that avoids grasping objects, favouring the walking behaviour in isolation and even possibly inhibiting it. Since the resulting option aims at maximising a combination of the cumulants $\ecm_i$, it can itself be seen as a combination of the options $o_{\ecm_i}$.

\section{Learning with combined options}
\label{sec:learning_combined_options}

\begin{wrapfigure}{r}{0.55\textwidth}
\vspace{-5mm}
\begin{center}
\includegraphics[scale=0.75]{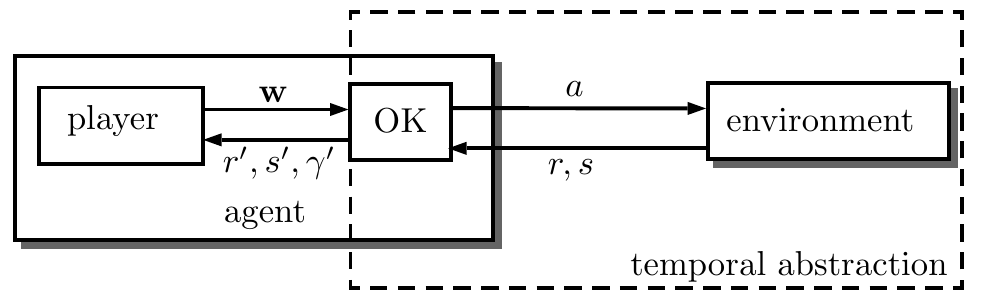}
\end{center}
\caption{OK mediates the interaction between player and environment. The exchange of information between OK and the environment happens at every time step. The interaction between player and OK only happens ``inside'' the agent when the termination action $\tau$ is selected by GPE and GPI (see Algorithms~\ref{alg:ok} and~\ref{alg:ok_player}).
 \label{fig:ok_diagram} }
\vspace{-5mm}
\end{wrapfigure}

Given a set of extended cumulants \Ec, in order to be able to combine the associated options using GPE and GPI one only needs the value functions $\Qec \defi \{\Qt^{\api_{\ecm_i}}_{\ecm_j} \; | \; \forall (i,j) \in \{1, 2, ..., d\}^2\}$.  
The set \Qec\ can be constructed using standard RL 
methods; for an illustration of how to do it with $Q$-learning see Algorithm~\ref{alg:build_ok} in App.~\ref{sec:pseudo_codes}. 

As discussed, once \Qec\ has been computed we can use GPE and GPI to synthesise options on the fly. In this case the newly-generated options are fully determined by the vector of weights $\w \in \R^d$. Conceptually, we can think of this process as an interface between an RL algorithm and the environment: the algorithm selects a vector $\w$, hands it over to GPE and GPI, and ``waits'' until the action returned by~(\ref{eq:gpi_ecm}) is the termination action $\tau$. Once $\tau$ has been selected, the algorithm picks a new $\w$, and so on. The RL method is thus interacting with the environment at a higher level of abstraction in which actions are combined skills defined by the vectors of weights \w. Returning to the analogy with a piano keyboard described in the introduction, we can think of each option $\api_{\ecm_i}$ as a ``key'' that can be activated by an instantiation of \w\ whose only non-zero entry is $w_i > 0$. Combined options associated with more general instantiations of \w\ would correspond to ``chords''. We will thus call the layer of temporal abstraction between algorithm and environment the \emph{option keyboard} (OK). We will also generically refer to the RL method interacting with OK as the ``player''. Figure~\ref{fig:ok_diagram} shows how an RL agent can be broken into a player and an OK.

Algorithm~\ref{alg:ok} shows a generic implementation of OK. Given a set of value functions \Qec\ and a vector of weights $\w$, OK will execute the actions selected by GPE and GPI until the termination action is picked or a terminal state is reached. During this process OK keeps track of the discounted reward accumulated in the interaction with the environment (line~\ref{it:sum_return}), which will be returned to the player when the interaction terminates (line~\ref{it:return}). As the options $\api_{\ecm_i}$ may depend on the entire trajectory since their initiation, OK uses an update function $u(h,a,s')$ that retains the parts of the history that are actually relevant for decision making (line~\ref{it:update_function}). For example, if OK is based on Markov options only, one can simply use the update function $u(h,a,s') = s'$.

The set \Qec\ defines a specific instantiation of OK; once an OK is in place any conventional RL method can interact with it as if it were the environment. As an illustration, Algorithm~\ref{alg:ok_player} shows how a keyboard player that uses a finite set of combined options $\W \defi \{\w_1, \w_2, ..., \w_n \}$ can be implemented using standard $Q$-learning by simply replacing the environment with OK. It is worth pointing out that if we substitute any other set of weight vectors $\W'$ for $\W$ we can still use the same OK, without the need to relearn the value functions in \Qec. We can even use sets of abstract actions \W\ that are infinite---as long as the OK player can deal with continuous action spaces~\cp{williams92simple,sutton2000policy,silver2014deterministic}.

Although the clear separation between OK and its player is instructive, in practice the boundary between the two may be more blurry. For example, if the player is allowed to intervene in all interactions between OK and environment, one can implement useful strategies like option interruption~\cp{sutton99between}. Finally, note that although we have been treating the construction of OK (Algorithm~\ref{alg:build_ok}) and its use (Algorithms~\ref{alg:ok} and \ref{alg:ok_player}) as events that do not overlap in time, nothing keeps us from carrying out the two procedures in parallel, like in similar methods in the literature~\cp{bacon2017option,vezhnevets2017feudal}. 

\vspace{-4mm}
\begin{tabular}{cc}
\begin{minipage}[c]{0.47\textwidth}
\begin{algorithm}[H]
   \caption{Option Keyboard (OK)} 
   \label{alg:ok}
\begin{algorithmic}[1]
\REQUIRE
$\left\{\begin{array}{ll}
s \in \S & \text{current state} \\
\w \in \R^d & \text{vector of weights} \\
\Qec\ & \text{value functions} \\ 
\gamma \in [0,1) & \text{discount rate}
\end{array}
\right.$
\STATE $h \la s$; $\;$ $r' \la 0$; $\;$ $\gamma' \la 1$
\REPEAT
\STATE {\small $a \la \argmax_{a'} \max_i [\sum_j w_j \Qt^{\api_{\ecm_i}}_{\ecm_j}(h,a') ]$}
\IF{$a \ne \tau$}
\STATE execute action $a$ and observe $r$ and $s'$
\STATE $r' \la r' + \gamma' r$ \label{it:sum_return}
\STATE {\bf if} $s'$ is terminal $\gamma' \la 0$ {\bf else} $\gamma' \la \gamma'\gamma$
\STATE $h \la u(h, a, s')$ \label{it:update_function}
\ENDIF
\UNTIL{$a = \tau$ {\bf or} $s'$ is terminal}
\STATE {\bf return} $s', r', \gamma'$ \label{it:return}
\vspace{1mm}
\end{algorithmic}
\end{algorithm}
\end{minipage}
&
\begin{minipage}[c]{0.49\textwidth}
\begin{algorithm}[H]
   \caption{$Q$-learning keyboard player} 
   \label{alg:ok_player}
\begin{algorithmic}[1]
\REQUIRE
$\left\{\begin{array}{ll}
\text{OK} & \text{option keyboard} \\
\W  & \text{combined options} \\
\Qec & \text{value functions} \\ 
\alpha, \epsilon, \gamma \in \R& \text{hyper-parameters} \\ 
\end{array}
\right.$
\STATE create $\Qt(s, \w)$ parametrised by $\params_Q$
\STATE select initial state $s \in \S$
\LOOP
\STATE {\bf if} Bernoulli($\epsilon$)=1 {\bf then} 
$\w \la$ Uniform($\W$) 
\STATE {\bf else} $\w \la \argmax_{\w' \in \W} \Qt(s, \w')$
\STATE $(s', r', \gamma') \la $ OK($s$, $\w$, \Qec, $\gamma$)  
\STATE $\delta \la r' + \gamma' \max_{\w'} \Qt(s', \w') - \Qt(s, \w)$ \label{it:td}
\STATE  {$\params_Q \la \params_Q + \alpha \delta \nabla_{\params_Q} \Qt(s,\w)$} 
\COMMENT{update $\Qt$ \label{it:learn_qt}}
\STATE {\bf if} $s'$ is terminal {\bf then } select initial $s \in \S$ 
\STATE {\bf else} $s \la s'$
\ENDLOOP
\end{algorithmic}
\end{algorithm}
\end{minipage}
\end{tabular}

\section{Experiments}
\label{sec:experiments}

We now present our experimental results illustrating the benefits of OK in practice. Additional details, along with further results and analysis, can be found in Appendix~\ref*{sec:details_experiments}.

\vspace{-3mm}
\subsection{Foraging world}
\label{sec:foraging}
\vspace{-2mm}

The goal in the foraging world is to manage a set of resources by navigating in a grid world and picking up items containing the resources in different proportions. For illustrative purposes we will consider that the resources are nutrients and the items are food. The agent's challenge is to stay healthy by keeping its nutrients within certain bounds. The agent navigates in the grid world using the four usual actions: up, down, left, and right. Upon collecting a food item the agent's nutrients are increased according to the type of food ingested. Importantly, the quantity of each nutrient decreases by a fixed amount at every step, so the desirability of different types of food changes even if no food is consumed. Observations are images representing the configuration of the grid plus a vector indicating how much of each nutrient the agent currently has (see Appendix~\ref{sec:details_foraging_world} for a technical description). 

What makes the foraging world particularly challenging is the fact that the agent has to \emph{travel} towards the items to pick them up, adding a spatial aspect to an already complex management problem. The dual nature of the problem also makes it potentially amenable to be tackled with options, since we can design skills that seek specific nutrients and then treat the problem as a management task in which actions are preferences over nutrients. However, the number of options needed can increase exponentially fast. If at any given moment the agent wants, does not want, or does not care about each nutrient, we need $3^m$ options to cover the entire space of preferences, where $m$ is the number of nutrients. This is a typical situation where being able to combine skills can be invaluable.

As an illustration, in our experiments we used $m=2$ nutrients and $3$ types of food. We defined a cumulant $\ecm_i \in \Ec$ associated with each nutrient as follows: $\ecm_i(h,a,s) = 0$ until a food item is consumed, when it becomes the increase in the associated nutrient. After a food item is consumed we have that $\ecm_i(h,a,s) = -\ind\{a \ne \tau\}$, where $\ind\{\cdot\}$ is the indicator function---this forces the induced option to terminate, and also illustrates how the definition of cumulants over histories $h$ can be useful (since single states would not be enough to determine whether the agent has consumed a food item). We used Algorithm~\ref{alg:build_ok} in Appendix~\ref{sec:pseudo_codes} to compute the $4$ value functions in \Qec. We then defined a $8$-dimensional abstract action space covering the space of preferences, $\W \defi \{-1,0,1\}^2 -\{ [0,0]\}$, and used it with the $Q$-learning player in Algorithm~\ref{alg:ok_player}. We also consider $Q$-learning using only the $2$ options maximizing each nutrient and a ``flat'' $Q$-learning agent that does not use options at all.
  
By modifying the target range of each nutrient we can create distinct scenarios with very different dynamics. Figure~\ref{fig:results_foraging_world} shows results in two such scenarios. Note how the relative performance of the two baselines changes dramatically from one scenario to the other, illustrating how the usefulness of options is highly context-dependent. Importantly, as shown by the results of the OK player, the ability to combine options in cumulant space makes it possible to synthesise useful behaviour from a given set of options even when they are not useful in isolation.

\begin{figure}
\centering
\vspace{-1mm}
\newcommand{\scl}{0.35}
\includegraphics[scale=\scl]{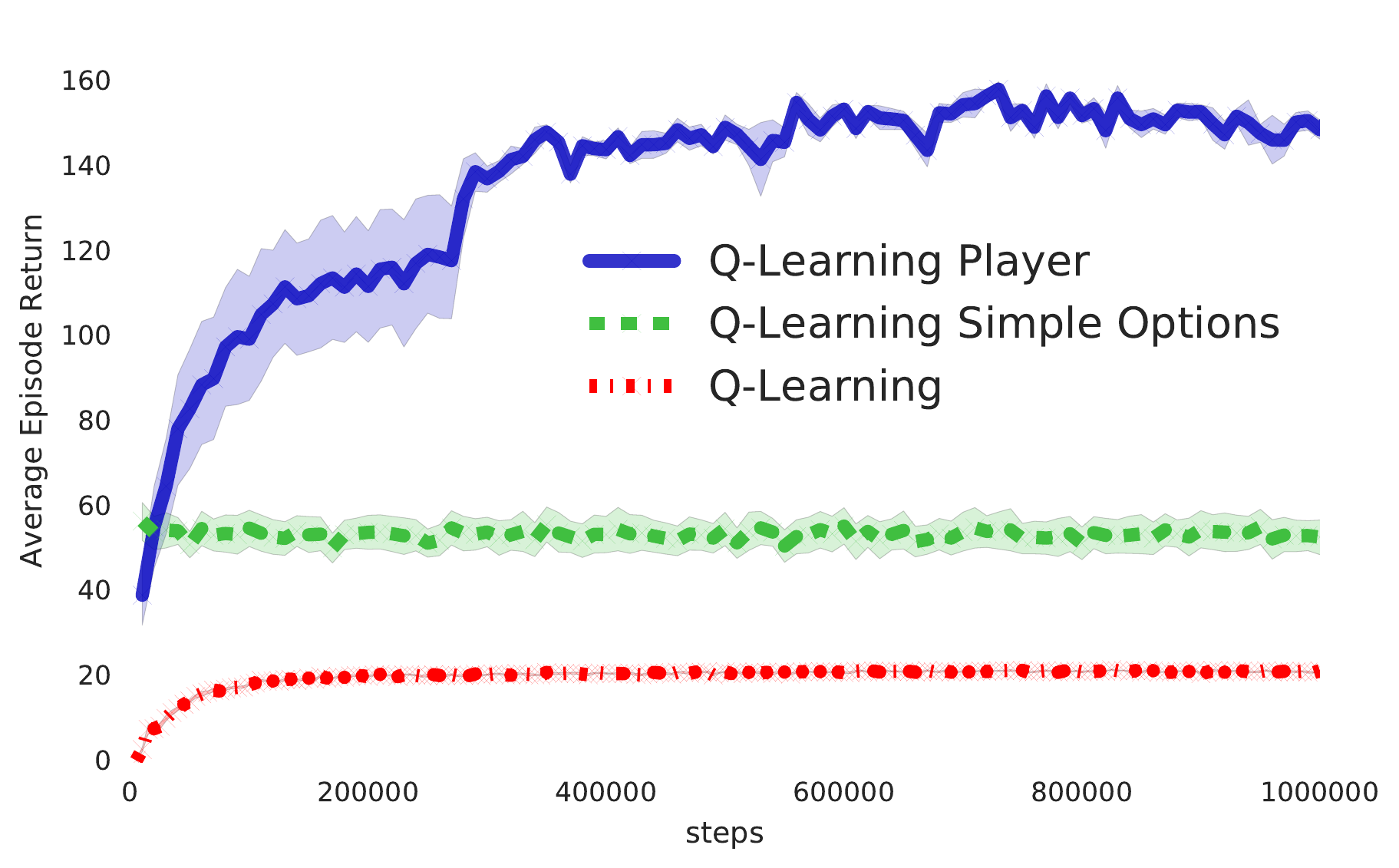}
\includegraphics[scale=\scl]{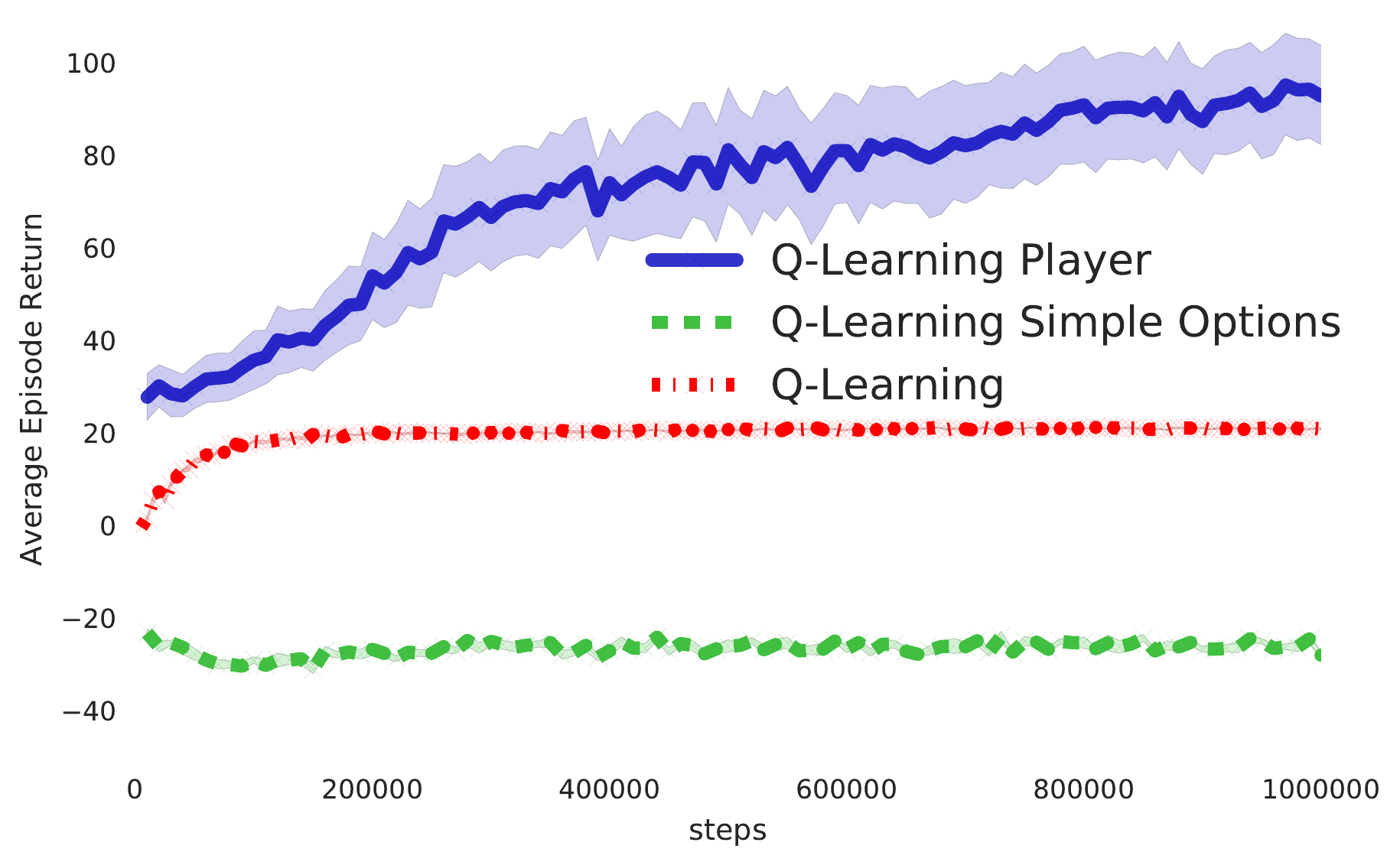}
\vspace{-3mm}
\caption{Results on the foraging world. The two plots correspond to different configurations of the environment (see Appendix~\ref{sec:details_foraging_world}). Shaded regions are one standard deviation over $10$ runs. \label{fig:results_foraging_world}}
\vspace{-4mm}
\end{figure}

\vspace{-2mm}
\subsection{Moving-target arena}
\label{sec:moving_target_arena}
\vspace{-2mm}

As the name suggests, in the moving-target arena the goal is to get to a target region whose location changes every time the agent reaches it. The arena is implemented as a square room with realistic dynamics defined in the MuJoCo physics engine~\cp{todorov2012mujoco}. The agent is a quadrupedal simulated robot with $8$ actuated degrees of freedom; actions are thus vectors in $[-1,1]^8$ indicating the torque applied to each joint~\cp{schulman2016high}. Observations are $29$-dimensional vectors with spatial and proprioception information (Appendix~\ref{sec:details_moving_target_arena}). The reward is always $0$ except when the agent reaches the target, when it is $1$. 
 
We defined cumulants in order to encourage the agent's displacement in certain directions. Let $\vv(h)$  be the vector of $(x,y)$ velocities of the agent after observing history $h$ (the velocity is part of the observation). Then, if we want the agent to travel at a certain direction \w\ for $k$ steps, we can define:
\begin{equation}
\label{eq:ecm_k_steps_policy}
\ecm_{\w}(h, a, \cdot) = \left\{\begin{array}{l}
                      \w^\t \vv(h) \text{ if } \length(h) \le k ; \\
                      -\ind\{a \ne \tau\} \text{ otherwise. }
                     \end{array}\right.
\end{equation}

The induced option will terminate after $k=8$ steps as a negative reward is incurred for all histories of length greater than $k$ and actions other than $\tau$. 
It turns out that even if a larger number of directions $\w$ is to be learned, we  only need to compute two value functions for each cumulant $\ecm_{\w}$. Since for all $\ecm_{\w} \in \Ec$ we have that $\ecm_{\w} = w_1 \ecm_{\vx} + w_2 \ecm_{\vy}$, where $\vx = [1,0]$ and $\vy = [0,1]$, we can use~(\ref{eq:gpe}) to decompose the value function of any option \api\ as $Q^{\api}_{\ecm_{\w}}(h,a) = w_1 Q^{\api}_{\ecm_{\vx}}(h,a) + w_2 Q^{\api}_{\ecm_{\vy}}(h,a) $. Hence, $|\Qec| = 2|\Ec|$,  resulting in a $2$-dimensional space \W\ in which $\w \in \R^2$ indicates the intended direction of locomotion. Thus, by learning a few options that move along specific directions, the agent is potentially able to synthesise options that travel in \emph{any} direction. 

For our experiments, we defined cumulants $\ecm_{\w}$ corresponding to the directions $0^o$, $120^o$, and $240^o$. To compute the set of value functions \Qec\ we used Algorithm~\ref{alg:build_ok} with $Q$-learning replaced by the deterministic policy gradient (DPG) algorithm~\cp{silver2014deterministic}. We then used the resulting OK with both discrete and continuous abstract-action spaces \W. For finite \W\ we adopted a $Q$-learning player (Algorithm~\ref{alg:ok_player}); in this case the abstract actions $\w_i$ correspond to $n \in\{4,6,8\}$ directions evenly-spaced in the unit circle. For continuous \W\ we used a DPG player. We compare OK's results with that of DPG applied directly in the original action space and also with $Q$-learning using only the three basic options.

Figure~\ref{fig:results_moving_target_arena} shows our results on the moving-target arena.  As one can see by DPG's results, solving the problem in the original action space is difficult because the occurrence of non-zero rewards may depend on a long sequence of lucky actions. When we replace actions with options we see a clear speed up in learning, even if we take into account the training of the options. If in addition we allow for combined options, we observe a significant boost in performance, as shown by the OK players' results. Here we see the expected trend: as we increase $|\W|$ the OK player takes longer to learn but achieves better final performance, as larger numbers of directional options allow for finer control.  

These results clearly illustrate the benefits of being able to combine skills, but how much is the agent actually using this ability? In Figure~\ref{fig:results_moving_target_arena} we show a histogram indicating how often combined options are used by OK to implement directions $\w \in \R^2$ across the state space (details in App.~\ref{sec:details_moving_target_arena}). As shown, for abstract actions \w\ close to $0^o$, $120^o$ and $240^o$ the agent relies mostly on the $3$ options trained to navigate along these directions, but as the intended direction of locomotion gets farther from these reference points combined options become crucial. This shows how the ability to combine skills can extend the range of behaviours available to an agent without the need for additional learning.\footnote{A video of the quadrupedal simulated robot being controlled by the DPG player can be found on the following link: \href{https://youtu.be/39Ye8cMyelQ}{https://youtu.be/39Ye8cMyelQ}.}

\begin{figure}
\centering
\includegraphics[scale=0.38]{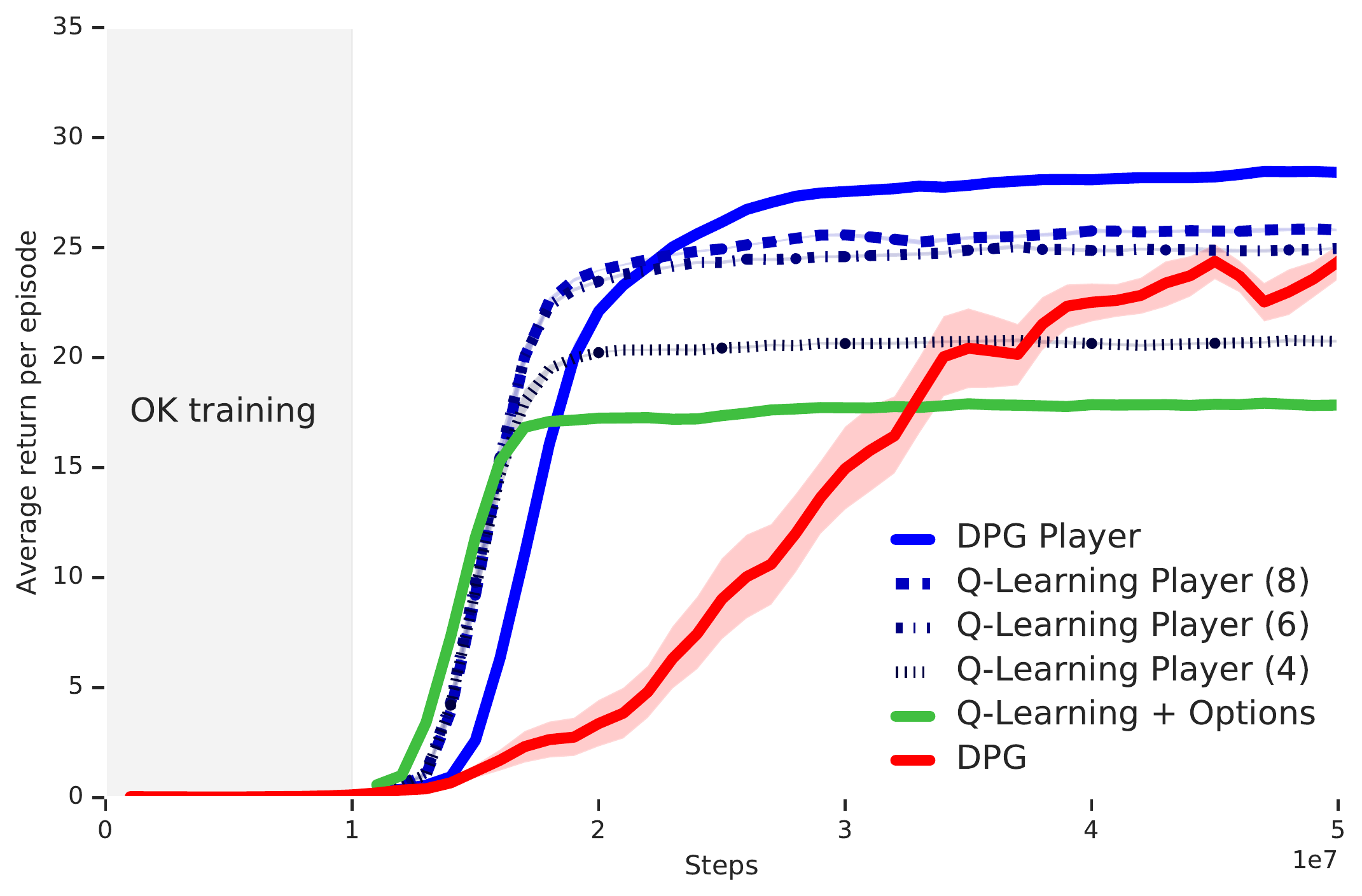}
 \raisebox{0.2\height}{\includegraphics[scale=0.17]{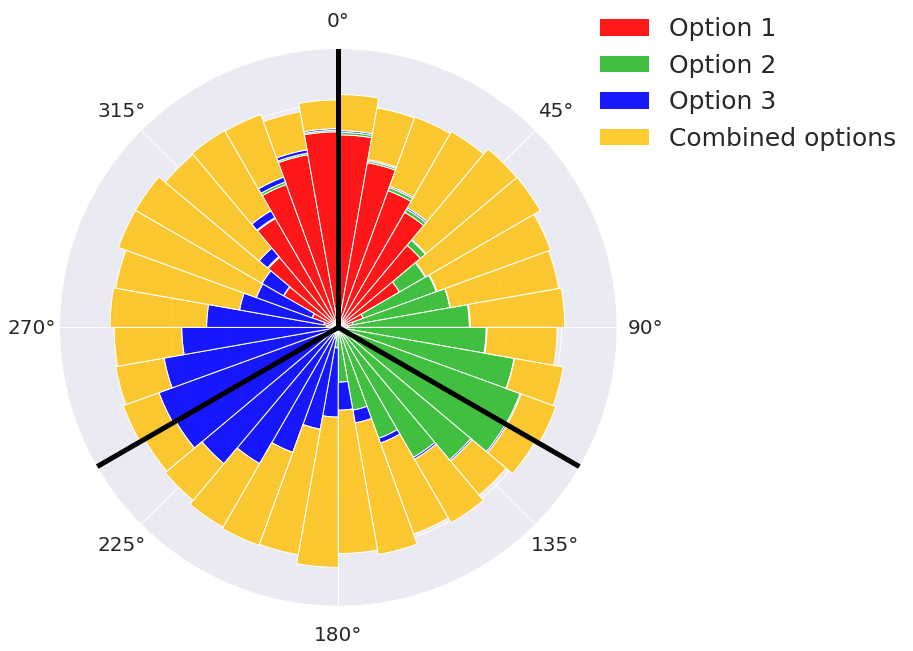}}
\caption{ \textbf{Left}: Results on the moving-target arena. All players used the \emph{same} keyboard, so they share the same OK training phase. Shaded regions are one standard deviation over $10$ runs.  \textbf{Right}: Histogram of options used by OK to implement directions $\w$. Black lines are the three basic options.  \label{fig:results_moving_target_arena}}
\vspace{-4mm}
\end{figure}

Even if one accepts the premise of this paper that skills should be combined in the space of cumulants, it is natural to ask whether other strategies could be used instead of GPE and GPI. Although we are not aware of any other algorithm that explicitly attempts to combine skills in the space of cumulants, there are methods that do so in the space of value functions~\cp{todorov2009compositionality,dasilva2009linear,haarnoja2018composable,hunt2018entropic}. \ct{haarnoja2018composable} propose a way of combining skills based on entropy-regularised value functions. Given a set of cumulants $\ecm_1, \ecm_2, ..., \ecm_d$, they propose to compute a skill associated with $\ecm = \sum_i w_i \ecm_i$ as follows: $\hat{\api}_\ecm(h) \in \argmax_{a \in \A^+} \sum_j w_j \hat{Q}^{\api_{\ecm_j}}_{\ecm_j} (h,a)$, where $\hat{Q}^{\api_{\ecm_j}}_{\ecm_j} (h,a)$ are entropy-regularised value functions and $w_j \in [-1,1]$. We will refer to this method as \emph{additive value composition} (AVC). 

How well does AVC perform as compared to GPE and GPI? In order to answer this question we reran the previous experiments but now using $\hat{\api}_\ecm(h)$ as defined above instead of the option $\tilde{\api}_\ecm(h)$ computed through~(\ref{eq:gpe_ecm}) and~(\ref{eq:gpi_ecm}). In order to adhere more closely to the assumptions underlying AVC, we also repeated the experiment using an entropy-regularised OK~ \cp{haarnoja2018soft} (App.~\ref{sec:details_moving_target_arena}). Figure~\ref{fig:comparison_moving_target_arena} shows the results. As indicated in the figure, GPE and GPI outperform AVC both with the standard and the entropy-regularised OK. A possible explanation for this is given in the accompanying polar scatter chart in Figure~\ref{fig:comparison_moving_target_arena}, which illustrates how much progress each method makes, over the state space, in all directions \w~(App.~\ref{sec:details_moving_target_arena}). The plot suggests that, in this domain, the directional options implemented through GPI and GPE are more effective in navigating along the desired directions (also see~\cp{hunt2018entropic}).

\begin{figure}
\centering
\includegraphics[scale=0.38]{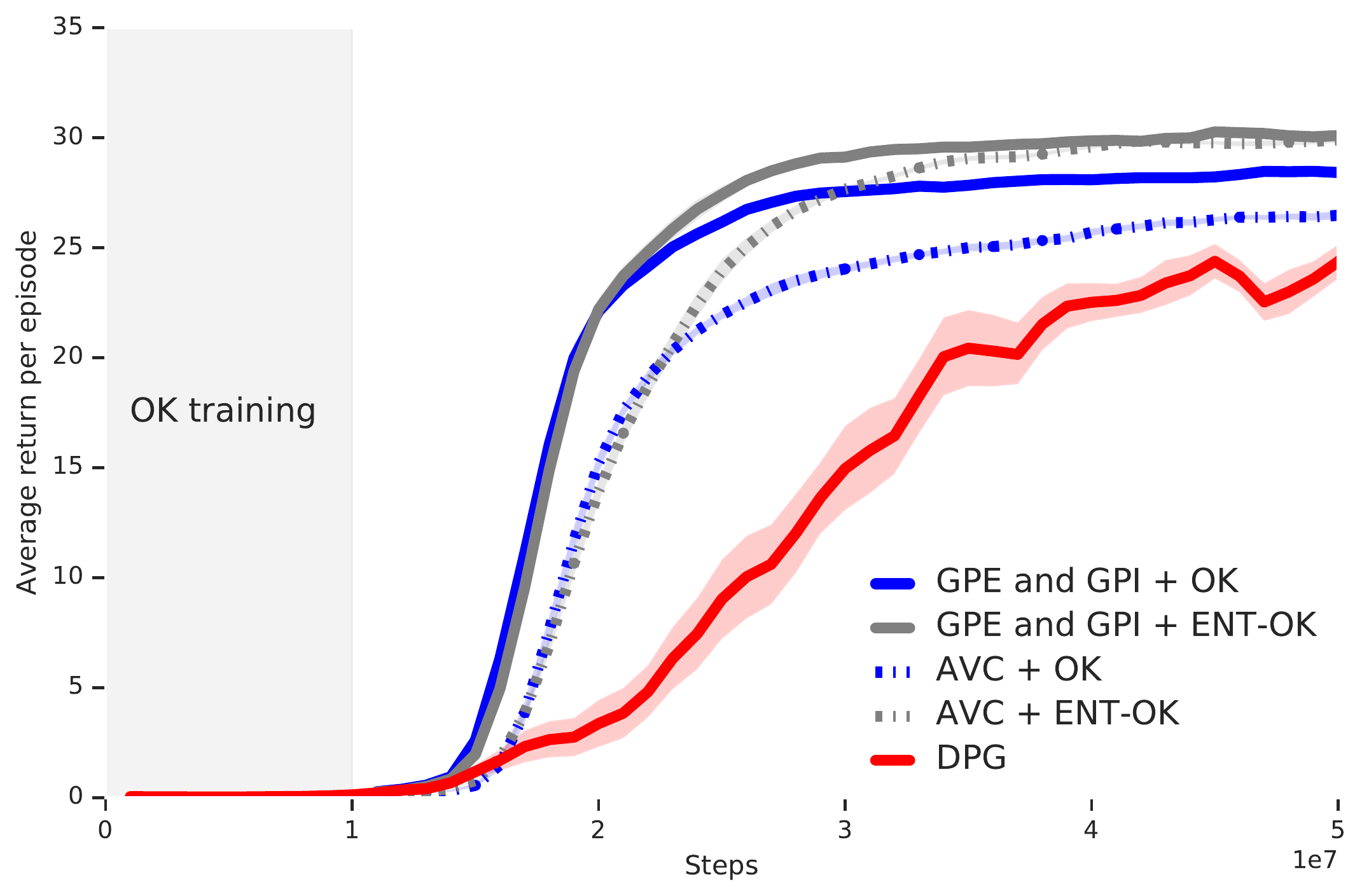}
 \raisebox{0.2\height}{\includegraphics[scale=0.17]{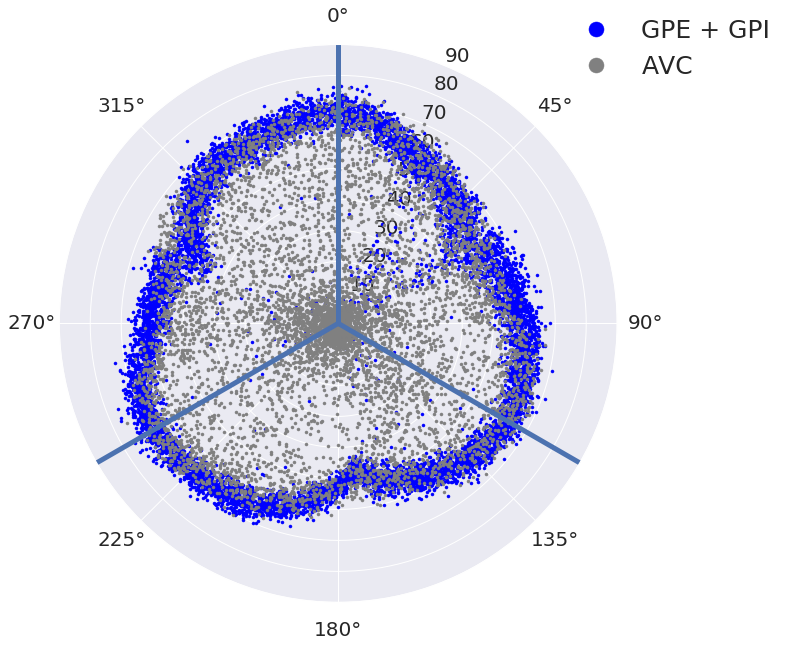}}
\vspace{-2mm}
 \caption{ \textbf{Left}: Comparison of GPE and GPI with AVC on the moving-target arena. Results were obtained by a DPG player using a standard OK and an entropy-regularised counterpart (ENT-OK). We trained several ENT-OK with different regularisation parameters and picked the one leading to the best AVC performance. The same player and keyboards were used for both methods. Shaded regions are one standard deviation over $10$ runs.  \textbf{Right}: Polar scatter chart showing the average distance travelled by the agent along directions $\w$ when combining options using the two competing methods.\label{fig:comparison_moving_target_arena}}
\vspace{-4mm}
\end{figure}

\vspace{-2mm}
\section{Related work}
\label{sec:related_work}
\vspace{-2mm}

Previous work has used GPI and successor features, the linear form of GPE considered here, in the context of transfer~\cp{barreto2017successor,barreto2018transfer,borsa2019universal}. A crucial assumption underlying these works is that the reward can be well approximated as  
$
\label{eq:linear_assum}
r(s,a,s') \approx \sum_i w_i c_i(s,a,s').
$
By solving a regression problem, the agent finds a $\w \in \R^d$ that leads to a good approximation of $r(s,a,s')$ and uses it to apply GPE and GPI (equations~(\ref{eq:gpe}) and~(\ref{eq:gpi}), respectively).  In terms of the current work, this is equivalent to having a keyboard player that is only allowed to play one endless ``chord''. Through the introduction of a termination action, in this work we replace policies with options that may eventually halt. 
Since policies are options that never terminate, the previous framework is a special case of OK. Unlike in the previous framework, with OK we can also \emph{chain} a sequence of options, resulting in more flexible behaviour. Importantly, this allows us to completely remove the linearity assumption on the rewards.

We now turn our attention to previous attempts to combine skills with no additional learning. As discussed, one way to do so is to work directly in the space of policies. Many policy-based methods first learn a parametric representation of a lower-level policy, $\pi(\cdot |\, s; \vtheta)$, and then use $\vtheta \in \R^d$ as the actions for a higher-level policy $\mu: \S \mapsto \R^d$~\cp{hees2016learning,frans2017meta2,haarnoja2018latent}. 
One of the central arguments of this paper is that combining skills in the space of cumulants may be advantageous because it corresponds to manipulating the \emph{goals} underlying the skills. This can be seen if we think of $\w \in \R^d$ as a way of encoding skills and compare its effect on behaviour with that of $\vtheta$: although the option induced by $\w_1 + \w_2$ through~(\ref{eq:gpe_ecm}) and~(\ref{eq:gpi_ecm}) will seek a combination of both its constituent's goals, the same cannot be said about a skill analogously defined as $\pi(\cdot |\, s;\, \vtheta_1 + \vtheta_2)$. More generally, though, one should expect both policy- and cumulant-based approaches to have advantages and disadvantages.

Interestingly, most of the previous attempts to combine skills in the space of value functions are based on entropy-regularised RL, like the already discussed AVC~\cp{ziebart2010modeling,fox2016taming,haarnoja2017reinforcement,haarnoja2018composable}. \ct{hunt2018entropic} propose a way of combining skills which can in principle lead to optimal performance if one knows in advance the weights of the intended combinations. They also extend GPE and GPI to entropy-regularised RL.  
\ct{todorov2007linearly} focuses on entropy-regularised RL on linearly solvable MDPs.  \ct{todorov2009compositionality} and \ct{dasilva2009linear} have shown how, in this scenario, one can compute optimal skills corresponding to linear combinations of other optimal skills---a property later explored by \ct{saxe2017hierarchy} to propose a hierarchical approach. Along similar lines, \ct{niekerk2019composing} have shown how optimal value function composition can be obtained in entropy-regularised shortest-path problems with deterministic dynamics, with the non-regularised setup as a limiting case.

 \vspace{-1mm}
\section{Conclusion}
\label{sec:conclusion}
 \vspace{-3mm}
 
The ability to combine skills makes it possible for an RL agent to learn a small set of skills and then use them to generate a potentially very large number of distinct behaviours. A robust way of combining skills is to do so in the space of cumulants, but in order to accomplish this one needs to solve two problems: (1) establish a well-defined mapping between cumulants and skills and (2) define a mechanism to implement the combined skills without having to learn them.

The two main technical contributions of this paper are solutions for these challenging problems. First, we have shown that every deterministic option can be induced by a cumulant defined in an extended domain. This novel theoretical result provides a way of thinking about options whose interest may go beyond the current work. Second, we have described how to use GPE and GPI to synthesise combined options on-the-fly, \emph{with no learning involved}. To the best of our knowledge, this is the only method to do so in general MDPs with performance guarantees for the combined options. 

We used the above formalism to introduce OK, an interface to an RL problem in which actions correspond to combined skills.  Since OK is compatible with essentially any RL method, it can be readily used to endow our agents with the ability to combine skills. In describing the analogy with a keyboard that inspired our work, \ct{sutton2016towards} calls for the need of ``something larger than actions, but more combinatorial than the conventional notion of options.'' We believe OK provides exactly that.

\newpage 

\section*{Acknowledgements}
We thank Joseph Modayil for first bringing the subgoal keyboard idea to our attention, and also for the subsequent discussions on the subject. We are also grateful to Richard Sutton, Tom Schaul, Daniel Mankowitz, Steven Hansen, and Tuomas Haarnoja for the invaluable conversations that helped us develop our ideas and improve the paper. Finally, we thank the anonymous reviewers for their comments and suggestions.

\bibliographystyle{abbrvnat}

\appendix
\newpage

\vspace{7mm} 
\begin{center}
\vspace{7mm} 
\noindent\makebox[\textwidth]{\rule{\textwidth}{2.0pt}} \\
\vspace{3mm} 
{\bf {\LARGE  The Option Keyboard \vspace{2mm} \\ Combining Skills in Reinforcement Learning} \\
\vspace{2mm} {\Large Supplementary Material} }
\vspace{5mm} 
\noindent\makebox[\textwidth]{\rule{\textwidth}{1.0pt}}

  {\bf Andr\'e Barreto}, 
  {\bf Diana Borsa}, 
  {\bf Shaobo Hou}, 
  {\bf Gheorghe Comanici}, 
  {\bf Eser Ayg\"un},    \\
  {\bf Philippe Hamel}, 
  {\bf Daniel Toyama}, 
  {\bf Jonathan Hunt}, 
  {\bf Shibl Mourad}, 
  {\bf David Silver}, 
  {\bf Doina Precup}  \vspace{1.0mm}  \\
   \texttt{\small \{andrebarreto,borsa,shaobohou,gcomanici,eser\}@google.com} \\
   \texttt{\small \{hamelphi,kenjitoyama,jjhunt,shibl,davidsilver,doinap\}@google.com}  \vspace{1.5mm} \\
  DeepMind 

\end{center}
\maketitle

\begin{abstract}
In this supplement we give details of the theory and experiments that had to 
be left out of the main paper due to space constraints. We prove our theoretical result, provide a thorough description of the protocol used to carry out our experiments, and present details of the algorithms. We also present additional empirical results and analysis, as well as a more in-depth discussion of several aspects of OK. The numbering of sections, equations, and figures resume from what is used in the main paper, so we refer to these elements as if paper and supplement were a single document.
\end{abstract}

\section{Theoretical results}
\label{sec:theory}

\setcounter{proposition}{0}
\begin{proposition}
Every extended cumulant induces at least one deterministic option, and every deterministic option can be unambiguously induced by an infinite number of extended cumulants.
\end{proposition}
\begin{proof}
We start by showing that every extended cumulant induces one or more deterministic options. Let $\ecm: \H \times \A^+ \times \S \mapsto \R$ be a cumulant defined in an MDP $M \defi (\S, \A, p, \cdot, \gamma)$; our strategy will be to define an extended MDP $M^+$ and a corresponding cumulant $\hat{\ecm}$ and show that maximising $\hat{\ecm}$ over $M^+$ corresponds to a deterministic option in the original MDP $M$. 

Since we need to model the termination of options, we will start by defining a fictitious absorbing state \nulls\ and let $\H^+ \defi \H \cup \{\nulls\}$. Moreover, we will use the following notation for the last state in a history: $\curr(h_{t: t+k}) = s_{t+k}$. Define $M^+  \defi (\H^+,\A^+, \hat{p},  \cdot, \gamma)$ where 
\begin{align*}
& \hat{p}(has | h,a) = p(s | \curr(h),a) \mbox{ for all } (h,a,s) \in \H \times \A \times \S, \\
& \hat{p}(\nulls | h, \tau) = 1 \mbox{ for all } h \in \H, \text{ and }\\
& \hat{p}(\nulls | \nulls, a) = 1 \mbox{ for all } a \in \A^+.
\end{align*}
We can now define the cumulant $\hat{\ecm}$ for $M^+$ as follows:
\begin{align*}
  &  \hat{\ecm}(h, a, s)  = \ecm(h, a, s) \mbox{ for all } (h,a,s) \in \H \times \A^+ \times \S, \mbox{ and} \\ 
  & \hat{\ecm}(\nulls, a, \nulls)  = 0\mbox{ for all } a \in \A^+.
\end{align*}
We know from the dynamic programming theory that maximising $\hat{\ecm}$ over $M^+$ has a unique optimal value function $Q^*_{\hat{\ecm}}$~\cp{puterman94markov}; we will use $Q^*_{\hat{\ecm}}$ to induce the three components that define an option.
First, define the option's policy $\pi_{\ecm}: \H \to \A$ with $\pi_{\ecm}(h) \defi \argmax_{a \ne \tau} Q^{*}_{\hat{\ecm}}(h,a)$ (with ties broken arbitrarily). Then, define the termination function as 
 \begin{equation*}
\beta_{\ecm}(h) \defi \left\{\begin{array}{l}
                      1 \text{ if } \tau  \in \argmax_a Q^{*}_{\hat{\ecm}}(h,a), \\
                      0 \text { otherwise. }
                     \end{array}\right.
 \end{equation*}
Finally, let $\I_{\ecm} \defi \{s \, | \, \beta_{\ecm}(s) = 0\}$ be the initiation set. It is easy to see that the option $o_{\ecm} \equiv (\I_{\ecm}, \pi_{\ecm}, \beta_{\ecm})$ is a deterministic option in the MDP $M$.

We now show that every deterministic option can be unambiguously induced by an infinite number of extended cumulants. Given a deterministic option $o$ specified in $M$ and a negative number $z <0$, our strategy will be to define an augmented cumulant $\ecm_z: \H \times \A^+ \times \S \mapsto \R$ that will induce option $o$ using the construction above (i.e. from the optimal value function $Q^{*}_{\hat{\ecm}_z}$ on $M^+$).
 
First we note a subtle point regarding the execution of options and the interaction between the initiation set and the termination function. 
Whenever an option $o$ is initiated in state $s \in \I_o$, it first executes $a=\pi_o(s)$ and only checks the termination function in the resulting state. This means that an option $o$ will always be executed for at least one time step. Similarly, an option that cannot be initiated in state $s$ does not need to terminate at  this state (that is, it can be that $s \notin \I_o$ and $\beta_o(h) < 1$, with $\curr(h)=s$). 
Given a deterministic option $o \defi (\I_o, \pi_o, \beta_o)$, let
\begin{equation}
\label{eq:api}
\ecm_z(h,a, \cdot) = \left\{\begin{array}{l}
                  0 \text{ if } a = \tau, h \in \S  \text{ and } h \notin \I_o, \\
                  0 \text{ if } a = \tau, h \notin \S  \text{ and } \beta_o(h) = 1, \\
                  0 \text{ if } a = \pi_o(h), \text{ and }\\
                  z  \text{ otherwise. }
                 \end{array}\right.
\end{equation}
We use the same MDP extension $M^+  \defi (\H^+, \A^+, \hat{p}, \cdot, \gamma)$ as described above and maximise the extended cumulant $\hat{e}_z$. It should be clear that $Q^*_{\hat{\ecm}_z}(h,a) = 0$ only when the action $a$ corresponds to either a transition or a termination dictated by option $o$, and $Q^*_{\hat{\ecm}_z}(h,a) < 0$ otherwise. As such, option $o$ is induced by the set of cumulants $\{ e_z \; | \; z < 0 \}$ of infinite size.
\end{proof}

\section{Additional pseudo-code}
\label{sec:pseudo_codes}

In this section we present one additional pseudo-code as a complement to the material in the main paper. Algorithm~\ref{alg:build_ok} shows a very simple way of building the set \Qec\ used by OK through $Q$-learning and $\epsilon$-greedy exploration. The algorithm uses fairly standard RL concepts. Perhaps the only detail worthy of attention is the strategy adopted to explore the environment, which switches between options with a given probability ($\epsilon_1$ in  Algorithm~\ref{alg:build_ok}). This is a simple, if somewhat arbitrary, strategy to collect data, which can probably be improved. It was nevertheless sufficient to generate good results in our experiments.

\begin{algorithm}
   \caption{Compute set \Qec\ with $\epsilon$-greedy $Q$-learning} 
   \label{alg:build_ok}
\begin{algorithmic}[1]
\REQUIRE 
$\left\{\begin{array}{ll}
\Ec = \{\ecm_1, \ecm_2, ..., \ecm_d\} & \text{cumulants} \\
\epsilon_1  & \text{probability of changing cumulant} \\
\epsilon_2  & \text{exploration parameter} \\
\alpha  & \text{learning rate} \\
\gamma  & \text{discount rate} \\
\end{array}
\right.$
\STATE select an initial state $s \in \S$
\STATE $k \la$ Uniform($\{1,2,..., d\}$) 
\REPEAT

\IF{Bernoulli($\epsilon_1$)=1}
\STATE $h \la s$
\STATE $k \la$ Uniform($\{1,2,..., d\}$) 
\COMMENT{pick a random $\ecm_k$ }
\ENDIF

\STATE {\bf if} Bernoulli($\epsilon_2$)=1 {\bf then} $a \la$ Uniform($\A$) 
\COMMENT{explore}
\STATE {\bf else} $a \la \argmax_b \Qt^{\api_{\ecm_k}}_{\ecm_k}(h,b)$
\COMMENT{GPI \label{it:gpi2}}
\IF{$a \ne \tau$}
\STATE execute action $a$ and observe $s'$
\STATE $h' \la u(h, a, s')$ 
\COMMENT{{\sl e.g.} $u(h,a,s') = h a s'$}

\FOR{$i \la 1, 2, ..., d$} 
\COMMENT{update $Q$-values}
\STATE $a' \la \argmax_{b} \Qt^{\api_{\ecm_i}}_{\ecm_i}(h', b)$
\COMMENT{$a' = \api_{\ecm_i}(h')$}
\FOR{$j \la 1, 2, ..., d$} 
\STATE {\small $\delta \la \ecm_j(h,a,s') + \gamma' \Qt^{\api_{\ecm_i}}_{\ecm_j}(h', a') - \Qt^{\api_{\ecm_i}}_{\ecm_j}(h,a)$}
\STATE {\small $\params_{\api_i} \la \params_{\api_i} + 
\alpha \delta \nabla_{\params_{\api_i}} \Qt^{\api_{\ecm_i}}_{\ecm_j}(h,a)$} 
\label{it:learn_q}
\ENDFOR
\ENDFOR

\STATE $s \la s'$

\ELSE 
\COMMENT{update values associated with termination}
\FOR{$i \la 1, 2, ..., d$}
\FOR{$j \la 1, 2, ..., d$} 
\STATE $\delta \la \ecm_j(h,\tau) - \Qt^{\api_{\ecm_i}}_{\ecm_j}(h,\tau)$
\STATE {\small $\params_{\api_i} \la \params_{\api_i} + 
\alpha \delta \nabla_{\params_{\api_i}} \Qt^{\api_{\ecm_i}}_{\ecm_j}(h,\tau)$} 
\label{it:learn_termination_bonus}
\ENDFOR
\ENDFOR
\ENDIF
\UNTIL{stop criterion is satisfied}
\STATE {\bf return} $\Qec \defi \{\Qt^{\api_{\ecm_i}}_{\ecm_j} \; | \; \forall (i,j) \in \{1, 2, ..., d\}^2\}$
\end{algorithmic}
\end{algorithm}

\section{Details of the experiments}
\label{sec:details_experiments}

In this section we give details of the experiments that had to be left out of the main paper due to the space limit.  

\subsection{Foraging world}
\label{sec:details_foraging_world}

\subsubsection{Environment}

We start by giving a more detailed description of the environment. The goal in the foraging world is to manage a set of $m$ ``nutrients'' $i$ by navigating in a grid world and picking up food items containing these nutrients in different proportions. At every time step $t$ the agent has a certain quantity of each nutrient available, $x_{it}$, and the desirability of nutrient $i$ is a function of $x_{it}$, $d_i(x_{it})$. For example, we can have $d_i(x_{it}) = 1$ if $x_{it}$ is within certain bounds and $d_i(x_{it}) = -1$ otherwise. The quantity $x_{it}$ decreases by a fixed amount $l_i$ at each time step, regardless of what the agent does: $x'_{it} = x_{it} - l_i$. The agent can increase $x_{it}$ by picking up one of the many food items available. Each item is of a certain type $j$, which defines how much of each nutrient it provides. We can thus represent a food type as a vector $\y_j \in \R^m$ where $y_{ji}$ indicates how much $x_{it}$ increases when the agent consumes an item of that type. If the agent picks up an item of type $j$ at time step $t$ it receives a reward of 
$
r_t = \sum_i y_{ji} d_i(x_{it}),
$
where $x_{it} = x'_{it-1} + y_{ji}$.
If the agent does not pick up any items it gets a reward of zero and $x_{it} = x'_{it-1}$.
The environment is implemented as a grid, with agent and food items occupying one cell each, and the usual four directional actions available. Observations are images representing the configuration of the grid plus a vector of nutrients $\mat{x}_t = [x_{1t}, ..., x_{mt}]$. 

The foraging world was implemented as a $12 \times 12$ grid with toroidal dynamics---that is, the grid ``wraps around'' connecting cells on opposite edges. We used $m=2$ nutrients and $3$ types of food: $\y_1 = (1,0)$, $\y_2 = (0,1)$, and $\y_3 = (1,1)$. Observations were image-like features of dimensions $12 \times 12 \times 3$, where the last dimension indicates whether there is a food item of a certain type present in a cell. The observations reflect the agent's ``egocentric'' view, {\sl i.e.}, the agent is always located at the centre of the grid and is thus not explicitly represented. At every step the amount available of each nutrient  was decreased by $l_i = 0.05$, for $i=1,2$. The desirability functions $d_i(x_i)$ used in the experiments of Section~\ref{sec:foraging} were:

\begin{center}
\begin{tabular}{ll}
\\
\multicolumn{1}{c}{{\bf Scenario 1}}  &  \multicolumn{1}{c}{{\bf Scenario 2}}  \\
$
d_1(x_1) = 
\begin{cases}
+1 & x_1\leq 10\\
-1 & x_1> 10
\end{cases}
$
&
$d_1(x_1) = 
\begin{cases}
+1 & x_1\leq 10\\
-1 & x_1< 10
\end{cases}
$
\\

$
d_2(x_2) = 
\begin{cases}
-1 & x_2\leq 5\\
+5 & 5<x_2<25\\
-1 & x_2\geq 25
\end{cases}
$

&

$d_2(x_2) = 
\begin{cases}
-1 & x_2\leq 5\\
+5 & 5<x_2<15\\
-1 & x_2\geq 15
\end{cases}
$ 
\end{tabular}
\end{center}

\subsubsection{Agent}

\paragraph{Agents' architecture:}
All the agents used a multilayer perceptron (MLP) with the same architecture to compute the value functions. The network had two hidden layers of $64$ and $128$ units with RELU activations. $Q$-learning's network had $|\A|=4$ output units corresponding to $\Qt(s,a)$. The network of the $Q$-learning player had $|\W|$ output units corresponding to $\Qt(s, w)$, while OK's network had $2 \times 2 \times |\A^{+}|$ outputs corresponding to $\Qt^{\api_{\ecm_i}}_{\ecm_j}(h,a) \in \Qec$.

The states $s$ used by $Q$-learning and the $Q$-learning player were $12 \times 12 \times 3$ images plus a two-dimensional vector \vx\ corresponding to the agent's nutrients. The histories $h$ used by OK were $s$ plus an indicator function signalling whether the agent has picked up a food item---that is, the update function $u(h,a,s')$ showing up in Algorithms~\ref{alg:ok} and~\ref{alg:build_ok} was defined as $u(h,a,s') = [\ind\{\text{agent has picked up a food item}\}, s']$.

\paragraph{Agents' training:} 
As described in Section~\ref{sec:foraging}, in order to build OK we defined one cumulant $\ecm_i \in \Ec$ associated with each nutrient. We now explain in more detail how cumulants were defined. If the agent picks up a food item of type $j$ at time step $t$, $\ecm_{i}(h_t, a_t, \cdot) = y_{ji}$. After a food item is picked up we have that $\ecm_i(h,a,s) = -\ind\{a \ne \tau\}$ for all $h$, $a$, and $s$---that is, the agent gets penalised unless it terminates the option. In all other situations $\ecm_{i} = 0$. 

OK was built using Algorithm~\ref{alg:build_ok} with the cumulants $\ecm_i \in \Ec$, exploration parameters $\epsilon_1=0.2$ and $\epsilon_2=0.1$, and discount rate $\gamma = 0.99$. The agent interacted with the environment in episodes of length $100$. We tried the learning rates $\alpha \in \mathcal{L}_1 \defi \{10^{-1}, 10^{-2}, 10^{-3}, 10^{-4}\}$ and selected the OK that resulted in the best performance using $\w = (1,1)$ on a scenario with $d_1(x) = d_2(x) = 1$ for all $x$. OK was trained for $5 \times 10^6$ steps, but visual inspection suggests that less than $10\%$ of the training time would lead to the same results.\footnote{Since the point of the experiment was not to make a case in favour of temporal abstraction, we did not deliberately try to minimise the total number of sample transitions used to train the options.} 

The $Q$-learning player was trained using Algorithm~\ref{alg:ok_player} with the abstract action set $\W \defi \{-1,0,1\}^2 -\{ [0,0]\}$ described in the paper, $\epsilon = 0.1$ and $\gamma = 0.99$.  All the agents interacted with the environment in episodes of length $300$. For all algorithms (flat $Q$-learning, $Q$-learning + options, and $Q$-learning player) we tried learning rates $\alpha$ in the set $\mathcal{L}_1$ above and picked the configuration that led to the maximum return averaged over the last $100$ episodes and $10$ runs.

\subsection{Moving-target arena}
\label{sec:details_moving_target_arena}

\subsubsection{Environment}

The environment was implemented using the MuJoCo physics engine~\cp{todorov2012mujoco} (see Figure~\ref{fig:moving_target_depiction}). The arena was defined as a bounded region $[-10,10]^2$ and the targets were circles of radius $0.8$. We used a control time step of $0.2$. The reward is always $0$ except when the agent reaches the target, when it gets a reward of $1$. In this case both the agent and the target reappear in random locations in $[-5,5]^2$.

\begin{figure}
\centering
\subfigure[Arena  \label{fig:ok_ant}]{
\includegraphics[height=50mm,width=50mm]{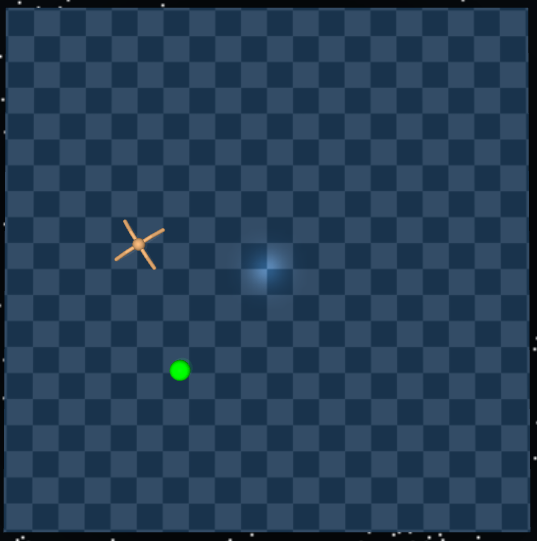}
}
\subfigure[Quadrupedal simulated robot  \label{fig:ok_ant_closeup}]{
\includegraphics[height=50mm,width=50mm]{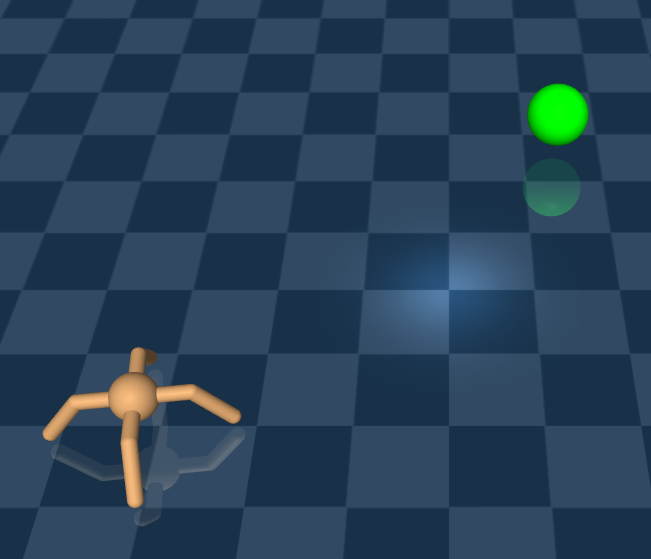}
}
\vspace{-3mm}
\caption{The moving-target arena \label{fig:moving_target_depiction}}
\end{figure}

\subsubsection{Agent}
\label{sec:details_moving_target_arena_agent}

\paragraph{Agents' architecture:}
The network architecture used for the agents was identical to that used in the experiments with the foraging world (Section~\ref{sec:details_foraging_world}). Observations are $29$-dimensional with the agent's current $(x,y)$ position and velocity, its orientation matrix ($3 \times 3$), a $2$-dimensional vector of distances from the agent to the current target, and two $8$-dimensional vectors with angles and velocities of each joint. The histories $h$ used by OK were simply the length of the trajectory plus the current state, that is, the update function $u(h,a,s')$ showing up in Algorithms~\ref{alg:ok} and~\ref{alg:build_ok} was defined in order to compute
$
h_{t: t+k} = [k, s_{t+k}].
$
As mentioned in the main paper, $\A \defi [-1,1]^8$.

\paragraph{Agents' training:}

The set of value functions \Qec\ used by OK was built  using Algorithm~\ref{alg:build_ok} with $Q$-learning replaced by deterministic policy gradient (DPG)~\cp{silver2014deterministic}. Specifically, for each cumulant $\ecm_i \in \Ec$ we ran standard DPG and used the same data to evaluate the resulting policies on-line over the set of cumulants \Ec. The cumulants $\ecm_i \in \Ec$ used were the ones described in Section~\ref{sec:moving_target_arena}, equation~(\ref{eq:ecm_k_steps_policy}).  During training exploration was achieved by adding zero-mean Gaussian noise with standard deviation $0.1$ to DPG's policy. We used batches of $10$ transitions per update, no experience replay, and a target network. The discount rate used was $\gamma = 0.9$. The agent interacted with the environment in episodes of length $300$. We swept over learning rates $\alpha \in \mathcal{L}_2 \defi \{10^{-2}, 10^{-3}, 3 \times 10^{-4}, 10^{-4}, 10^{-5}\}$, and selected the OK that resulted in the best performance in a small set of evaluation vectors $\w$ with $w_1, w_2 > 0$ (that is, the directions used for evaluation did not correspond to those of the basic options $\api_{\ecm_i}$). OK was trained for $10^7$ steps.

The $Q$-learning players were trained using Algorithm~\ref{alg:ok_player} with the discrete abstraction action set \W\ described in the paper, $\epsilon = 0.1$, and $\gamma = 0.99$. Updates were applied to batches of $10$ sample transitions. The DPG player was trained using the same implementation as the DPG used to build OK, and the same value $\gamma=0.99$ used by the $Q$-learning player. Note that, given a fixed $\w \in \R^2$, in order to compute the $\max$ operator appearing in~(\ref{eq:gpi_ecm}) we need to sample actions $\mat{a} \in \R^8$. We did so using a simple cross-entropy Monte-Carlo method with $50$ samples~\cp{deboer2005tutorial}. The same DPG implementation was also used by the flat DPG as a baseline, the only difference being that it used the actions space $\A \subset \R^8$ instead of the abstract action space $ \W \subset \R^2$.  All the agents interacted with the environment in episodes of length $1200$. For all algorithms ($Q$-learning, $Q$-learning player, DPG,  and DPG player) we tried learning rates $\alpha$ in the set $\mathcal{L}_2$ above and picked the configuration that led to the maximum average return, averaged over $10$ runs.

The results comparing GPE and GPI with AVC shown in Figure~\ref{fig:comparison_moving_target_arena} were generated exactly as explained above. In order to train the entropy-regularised OKs we used the soft actor-critic algorithm proposed by~\ct{haarnoja2018soft}. We trained one OK for each regularisation parameter in $\{0.001,0.01,0.03,0.1,0.3,1.0\}$ and selected the one leading to the best performance of AVC. In addition to the implementation of the AVC option $\hat{\api}_\ecm(h)$ described in Section~\ref{sec:moving_target_arena}, we also tried a ``soft'' version in which $\hat{\api}_\ecm$ is a stochastic policy defined as $\hat{\api}_\ecm(a|h) \propto \sum_j w_j \hat{Q}^{\api_{\ecm_j}}_{\ecm_j} (h,a)$, where, as before,  $\hat{Q}^{\api_{\ecm_j}}_{\ecm_j} (h,a)$ are entropy-regularised value functions and $w_j \in [-1,1]$. The results with the stochastic policy were slightly worse than the ones shown in Figure~\ref{fig:comparison_moving_target_arena} (this is consistent with some experiments reported by~\ct{haarnoja2018soft}).

\subsubsection{Experiments}

In order to generate the histogram shown in Figure~\ref{fig:results_moving_target_arena} we sampled $100\,000$ directions $\w \in \R^2$ from a player with a uniformly random policy and inspected the value of the action $a \in \R^8$ returned by GPI~(\ref{eq:gpi_ecm}). Specifically, we considered the selected action came from one of the three basic options $\api_{\ecm_i}$ if
\begin{equation}
\label{eq:comp_actions}
\min_{i \in \{1,2,3\}, \mat{a} \in \tilde{\A}} \left|Q^{\api_{\ecm_i}}_{\ecm} (h,\api_{\ecm_i}(h) - Q^{\api_{\ecm_i}}_{\ecm} (h,\mat{a}))\right| \le 0.15, 
\end{equation}
where $\ecm = \sum_i w_i \ecm_i$ and $\tilde{\A}$ is the set of actions sampled  through the cross-entropy sampling process described in the previous section. If~(\ref{eq:comp_actions}) was true we considered the action selected came from the option $\api_{\ecm_i}$ associated with the index $i$ that minimises the left-hand side of ~(\ref{eq:comp_actions}); otherwise we considered the action came from a combined option.

In order to generate the polar scatter chart shown in Figure~\ref{fig:comparison_moving_target_arena} we sampled $10\,000$ pairs $(s,\w)$, with $s$ sampled uniformly at random from \S\ and abstract actions \w\ sampled from an isotropic Gaussian distribution in $\R^d$ with unit variance (where $d=3$ for AVC and $d=2$ for GPE and GPI).\footnote{As explained in Section~\ref{sec:moving_target_arena}, in our implementation of GPE and GPI we explored the fact that $Q^{\api}_{\ecm_{\w}}(h,a) = w_1 Q^{\api}_{\ecm_{\vx}}(h,a) + w_2 Q^{\api}_{\ecm_{\vy}}(h,a)$ for any option \api\ and any direction $\w \in \R^2$, where $\vx = [1,0]$ and $\vy = [0,1]$, to only compute two value functions per cumulant $\ecm$. This results in a two-dimensional abstract space $\w \in \R^2$. Since GPE is not part of AVC (that is, the option induced by a cumulant is not evaluated under other cumulants), it is not clear how to carry out a similar decomposition in this case. } Then, for each pair $(s,\w)$, we ran the option resulting from~(\ref{eq:gpe_ecm}) and~(\ref{eq:gpi_ecm}) for $60$ simulated seconds, without termination, and measured the distance travelled along the desired direction $\w$ (for $\w \in \R^3$ we first projected the weights onto $\R^2$ using the decomposition discussed in Section~\ref{sec:moving_target_arena}). Each point in the scatter chart defines a vector whose direction is the intended \w\ and whose magnitude is the travelled distance along that direction.

\section{Discussion}
\label{sec:discussion}

In this section we take a closer look at some aspects of OK. We start with a thorough discussion on how extended cumulants can be used to define deterministic options; we then analyse several properties of GPE and GPI in more detail.

\subsection{Defining options through extended cumulants}
\label{sec:defining_options}

We have shown that every deterministic option $o$ can be represented by an augmented policy $\api_{\ecm}: \H \mapsto \A^+$, which in turn can be induced by an extended cumulant $\ecm: \H \times \A^+ \times \S \mapsto \R$ (in fact, by an infinite number of them). In order to provide some intuition on these relations, in this section we give a few concrete examples of how to generate potentially useful options using extended cumulants.

We start by defining an option that executes a policy $\pi: \S \mapsto \A$ for $k$ time steps and then terminates. This can be accomplished using the following cumulant:
\begin{equation}
\label{eq:ecm_k_steps_basic_policy}
\ecm(h, a, \cdot) = \left\{\begin{array}{l}
                      0 \text{ if } \length(h) \le k \text{ and } a = \pi(\curr(h)); \\
                      0 \text{ if } \length(h) = k + 1 \text{ and } a = \tau; \\
                      -1 \text { otherwise, }
                     \end{array}\right.
\end{equation}
where $\length(h)$ is the length of history $h$, that is, $\length(h_{t,t+k}) = k+1$, and $\curr(h_{t: t+k}) = s_{t+k}$  (also see~(\ref{eq:ecm_k_steps_policy})). Note that if $\pi(s) = a$ for all $s \in \S$ action $a$ is repeated $k$ times in sequence; for the particular case where $k=1$ we recover the primitive action $a$. Another instructive example is an option that navigates to a goal state $g \in \S$ and terminates once this state has been reached. We can get such behaviour using the following extended cumulant:
\begin{equation}
\label{eq:ecm_goal}
\ecm(h, a, \cdot) = \left\{\begin{array}{l}
                      1 \text{ if } \curr(h) = g \text{ and } a = \tau; \\
                      0 \text { otherwise. }
                     \end{array}\right.
\end{equation}
Note that the cumulant is non-zero only when the agent chooses to terminate in $g$. 
Yet another possibility is to define a fixed termination bonus $\ecm(h,\tau) = z$ for all $h \in \H$, where $z \in \R$; in this case the option will terminate whenever it is no longer possible to get more than $z$ discounted units of \ecm. 

Even though working in the space of histories \H\ is convenient at the conceptual level, in practice the extended cumulants only have to be defined in a small subset of this space, which makes them easy to be implemented. In order to implement~(\ref{eq:ecm_k_steps_basic_policy}), for example, one only needs to keep track of the number of steps executed by the option and the last state and action experienced by the agent ({\sl cf.} Section~\ref{sec:details_moving_target_arena_agent}). The implementation of~(\ref{eq:ecm_goal}) is even simpler, requiring only the current state and action. Obviously, one is not restricted to cumulants of these forms; other versions of $\ecm$ can define interesting trade-offs between terminating and continuing. 

As a final observation, note that, unlike with standard termination functions $\beta_o(h)$, (\ref{eq:termination}) depends on the value function $Q^{\api_{\ecm}}_{\ecm}$. This means that, when $Q^{\api_{\ecm}}_{\ecm}$ is being \emph{learned}, the termination condition may change during learning.  This can be seen as a natural way of incorporating $\beta_o(h)$ into the learning process, and thus impose a form of consistency on the agent's behaviour. When we define~(\ref{eq:termination}), we are asking the agent to terminate in $h$ if it cannot get more than $\ecm(h,\tau)$ (discounted) units of $\ecm$; thus, even if it \emph{is} possible to do so, a sub-optimal agent that is not capable of achieving this should perhaps indeed terminate. 

\subsection{GPE and GPI}

{\bf The nature of GPE and GPI's options}: Given a set of cumulants \Ec, GPE and GPI can be used to compute an approximation of any option induced by a linear combination of the elements of this set. Although this potentially gives rise to a very rich set of behaviours, not all useful combinations of skills can be represented in this way. To illustrate this point, suppose that all cumulants $\ecm \in \Ec$ take values in $\{0,1\}$. In this case, when the weights $\w$ are nonnegative, it is instructive to think of GPE and GPI as implementing something in between the {\tt AND} and the {\tt OR} logical operators, as positive cumulants are rewarding in isolation but more so in combination. GPE and GPI cannot implement a strict {\tt AND}, for example, since this would require only rewarding the agent when all cumulants are equal to $1$. \ct{niekerk2019composing} present a related discussion in the context of entropy-regularised~RL.   

{\bf The mechanics of GPE and GPI}: There are two ways in which OK's combined options can provide benefits with respect to an agent that only uses single options. As discussed in Section~\ref{sec:combining_options}, a combined option constructed through GPE and GPI \emph{can be different from all its constituent options}, meaning that the actions selected by the former may not coincide with any of the actions taken by the latter (including  termination). But, even when the combined option could in principle be recovered as a sequence of its constituents, having it can be very advantageous for the agent. To see why this is so, it is instructive to think of GPE and GPI in this case as a way of automatically carrying out an alternation of the single options that would otherwise have to be deliberately implemented by the agent. This means that, in order to emulate combined options that are a sequence of single options, a termination should occur at every point where the option achieving the maximum in~(\ref{eq:gpi_ecm}) changes, resulting in potentially many more decisions to be made by the agent.

{\bf Option discovery}: As discussed in Section~\ref{sec:combining}, the precise interface to an RL problem provided by OK is defined by a set of extended cumulants \Ec\ plus a set of abstract actions \W. A natural question is then how to define \Ec\ and \W. Although we do not have a definite answer to this question, we argue that these definitions should aim at exploiting a specific structure in the RL problem. Many RL problems allow for a hierarchical decomposition in which decisions are made at different levels of temporal abstraction. For example, as illustrated in Section~\ref{sec:moving_target_arena}, in a navigation task it can be beneficial to separate decisions at the level of intended locomotion ({\sl e.g.}, ``go northeast'') from their actual implementation ({\sl e.g.}, ``apply a certain force to a specific joint''). Most hierarchical RL algorithms exploit this sort of structure in the problem; another type of structure that has received less attention occurs when each hierarchical level can be further decomposed into distinct skills that can then be combined (for example, the action ``go northeast'' can be decomposed into ``go north'' and ``go east''). In this context, the cumulants in \Ec\ should describe the basic skills to be combined and the set \W\ should identify the combinations of these skills that are useful. Thus, the definition of \Ec\ and \W\ decomposes the problem of option discovery into two well-defined objectives, which can potentially make it more approachable.

{\bf The effects of approximation}: Once \Ec\ and \W\ have been defined one has an interface to a RL problem composed of a set of deterministic options $\tilde{\api}_{\ecm}$. Each $\tilde{\api}_{\ecm}$ is an \emph{approximation} of the option $\api_{\ecm}$ induced by the cumulant $e = \sum_i w_i \ecm_i$. \ct{barreto2017successor} have shown that it is possible to bound $Q^{\api_{\ecm}}_{\ecm} - Q^{\tilde{\api}_{\ecm}}_{\ecm}$ based on the quality of the approximations $\Qt^{\api_{\ecm_i}}_{\ecm_j}$ and the minimum distance between $\ecm$ and the cumulants $ \ecm_i \in \Ec$. Although this is a reassuring result, in the scenario studied here the sub-optimality of the options $\tilde{\api}_{\ecm}$ is less of a concern because it can potentially be circumvented by the operation of the player. To see why this is so, note that a navigation option that slightly deviates from the intended direction can be corrected by the other directional options (especially if it is prematurely interrupted, which, as discussed in Section~\ref{sec:defining_options}, can be a positive side effect of using~(\ref{eq:termination})). Although it is probably desirable to have good approximations of the options intended in the design of \Ec, the player should be able to do well as long as the set of available options is expressive enough. This suggests that the potential to induce a \emph{diverse} set of options may be an important criterion in the definition of the cumulants \Ec, something previously advocated in the literature.

\section{Additional results and analysis}

We now present some additional empirical results that had to be left out of the main paper due to the space limit.

\subsection{Foraging World}

In this section we will take a closer look at the results presented in the main paper and study step-by-step the behaviour induced by different desirability profiles of the nutrients. For each of these regimes of desirabily, we will also inquire what the combined options will do and which of them one would expect to be useful. In order to do this, we are going to consider various instantiations of the absract action set $\W$. As a reminder, the $Q$-learning player presented in the main paper was trained using Algorithm~\ref{alg:ok_player} with the abstract action set $\W \defi \{-1,0,1\}^2 -\{ [0,0]\}$. In the following section, we will refer back to this agent as 'Q-learning player (8)', indicating the cardinality of the set of absract action considered -- in this case, $|\W| = 8$. In addition, we will consider in our investigations $\W_0 \defi \{(1,0), (0,1)\}$, the set of basic options, and the following individual combinations: $\w_1 = (1,1)$, $\w_2 = (1,-1)$, $\w_3 = (-1,1)$, and $\w_4 = (-1,-1)$. As in the main paper, we refer to the instantiation of the $Q$-learning player that uses $\W_0$ as \textit{Q-learning player + options} (QO). Otherwise, we will use QP($n$) to refer to a Q-learning player with $n$ (combined) options. Specifically, we adopt QP($3$)-$i$ to refer to the players using $\W_i \defi \W_0 \cup \{\w_i\}$. Finally, throughout this study, we include a version of $Q$-learning (QL) that uses the original action space to serve as our flat agent baseline. This agent does not use any form of abstraction and does not have access to the trained options.

Most of the settings of the environment stay the same: we are going to be considering two types of nutrients and three food items $\{\y_1 = (1,0), \y_2 = (0,1), \y_3 = (1,1)\}$ available for pick up. The only thing we are going to be varying is the desirability function associated with each of these nutrients. We will see that this alone already gives rise to very interesting and qualitatively different learning dynamics. 
In particular, we are going to be looking at four scenarios, slightly simpler than the ones used in the paper, based on the same (pre-trained) keyboard $\Qec$. These scenarios should help the reader to build some intuition for what a player based on this keyboard could achieve by combining options under multiple changes in the desirability functions---as exemplified by the scenarios 1 and 2 in the main paper, Figure \ref{fig:results_foraging_world}. 

\begin{figure*}
\centering
\newcommand{\scl}{0.365}
\subfigure[Scenario $A1$  \label{fig:add_results_foraging_world0}]{
\includegraphics[scale=\scl]{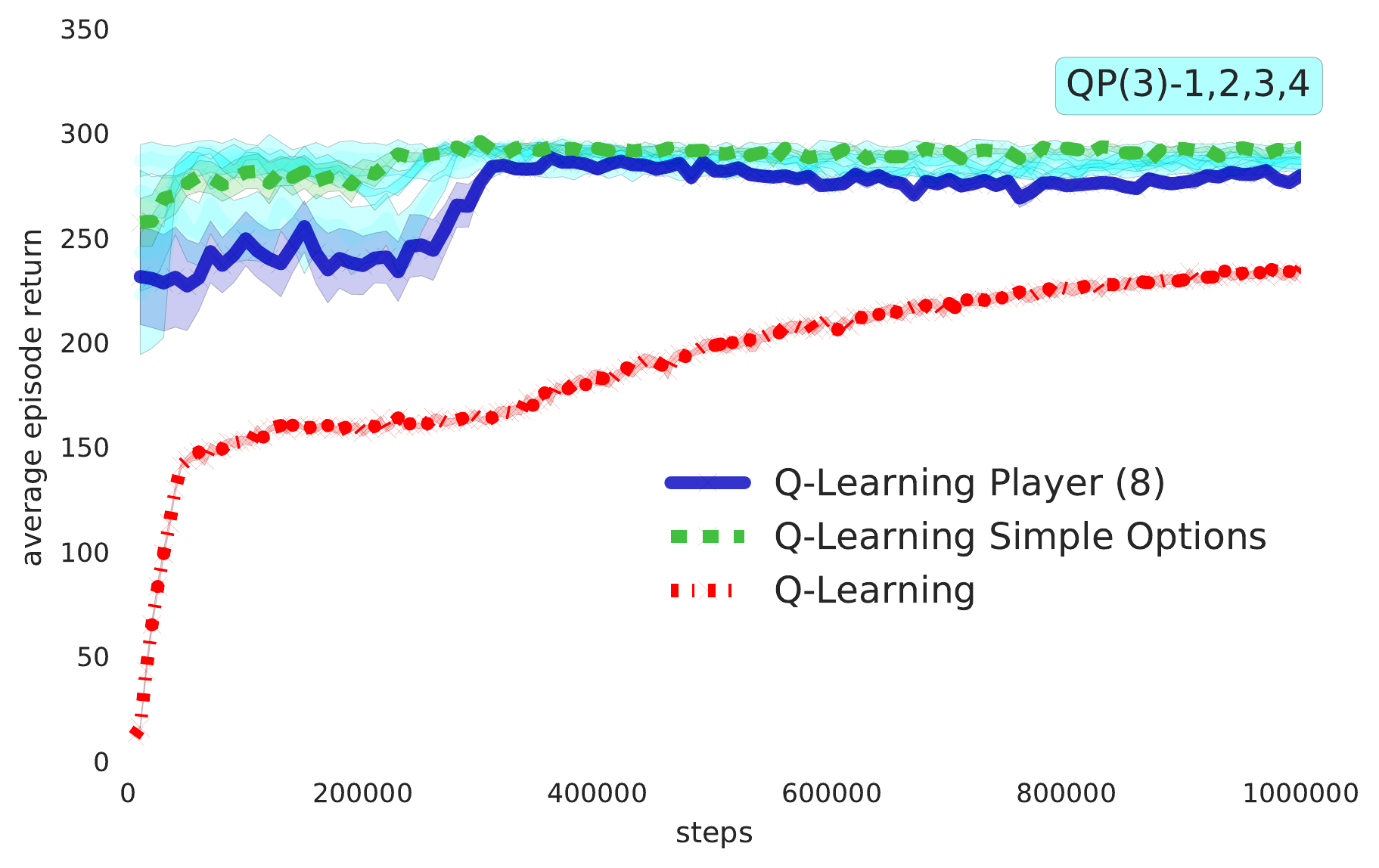}
}\hfill
\subfigure[Scenario $A2$  \label{fig:add_results_foraging_world1}]{
\includegraphics[scale=\scl]{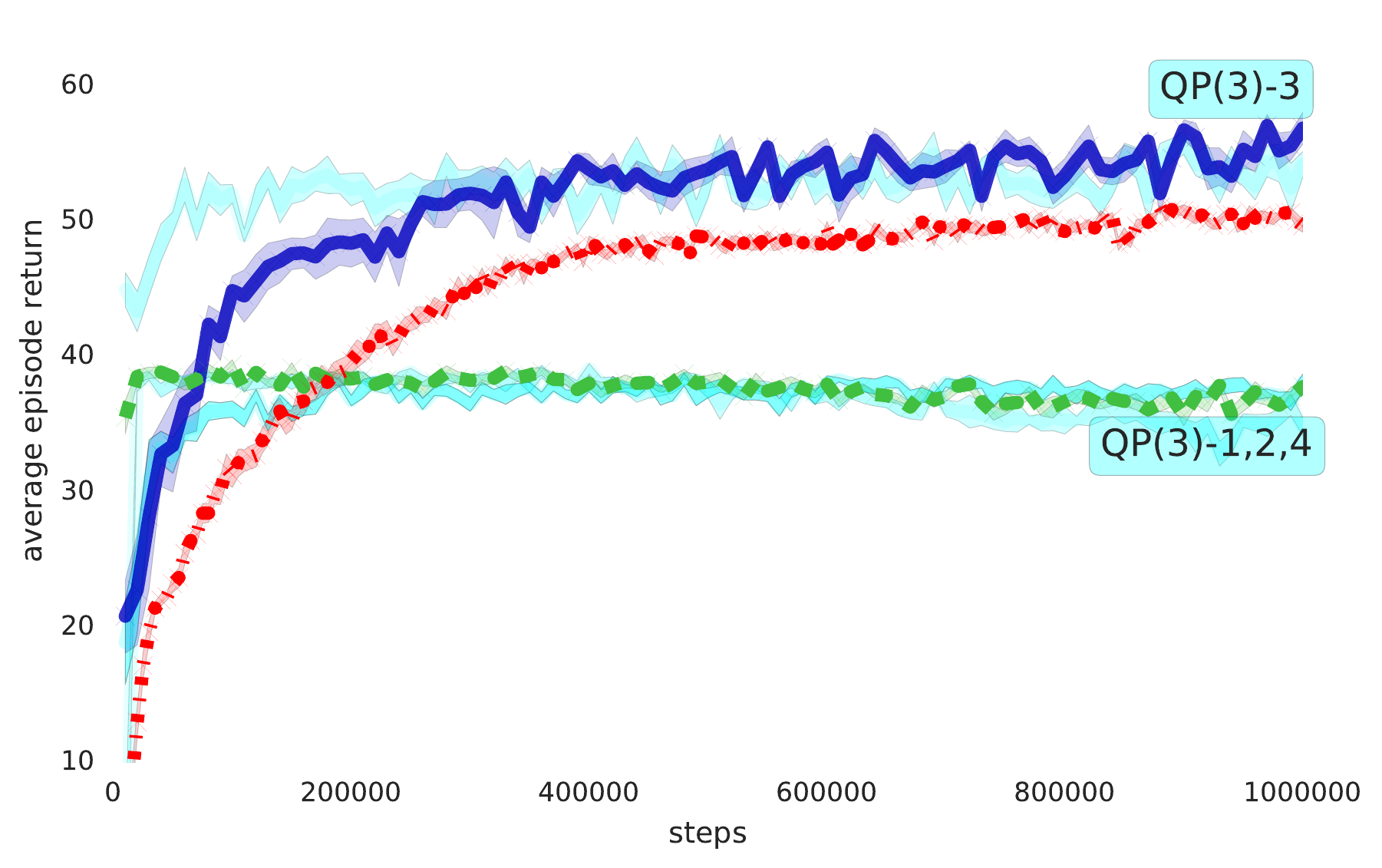}
}

\subfigure[Scenario $A1$: $d_i(x_i)$  \label{fig:add_foraging_world0}]{
\includegraphics[width=0.44\textwidth]{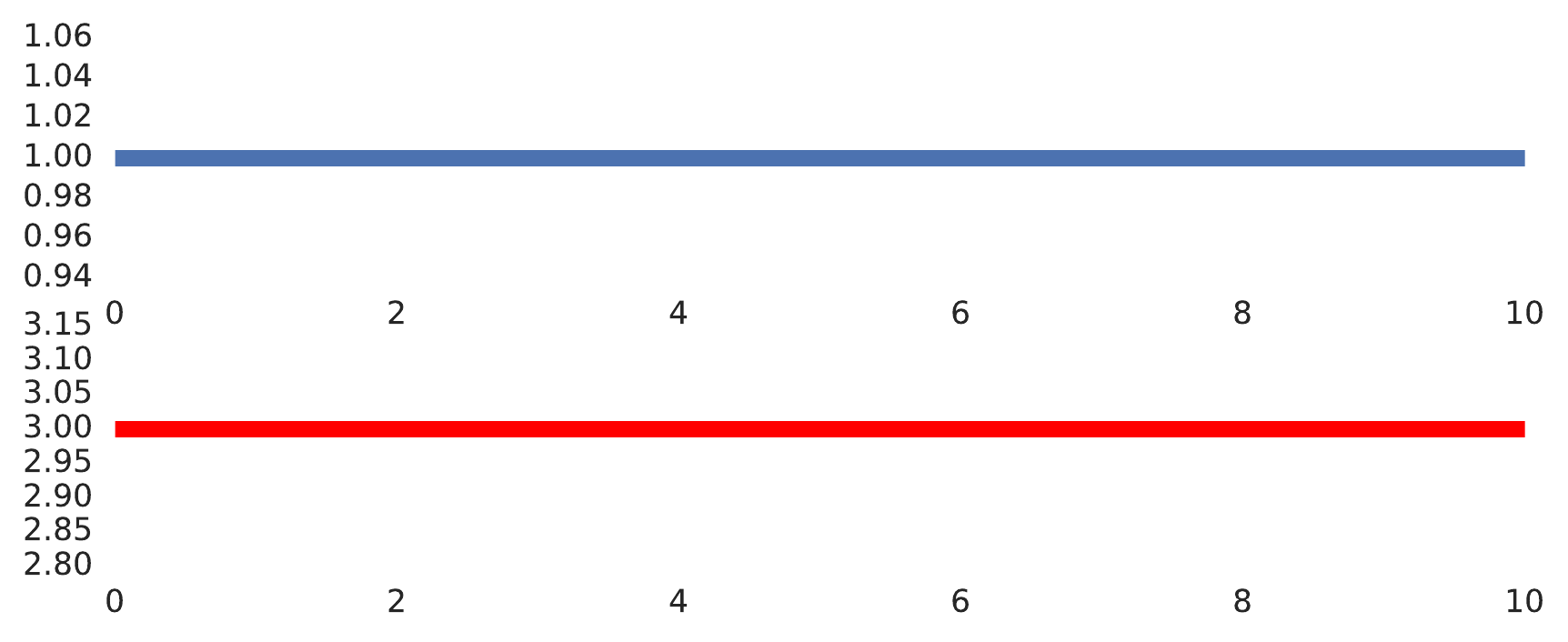}
} 
\subfigure[Scenario $A2$: $d_i(x_i)$  \label{fig:add_foraging_world1}]{
\includegraphics[width=0.44\textwidth]{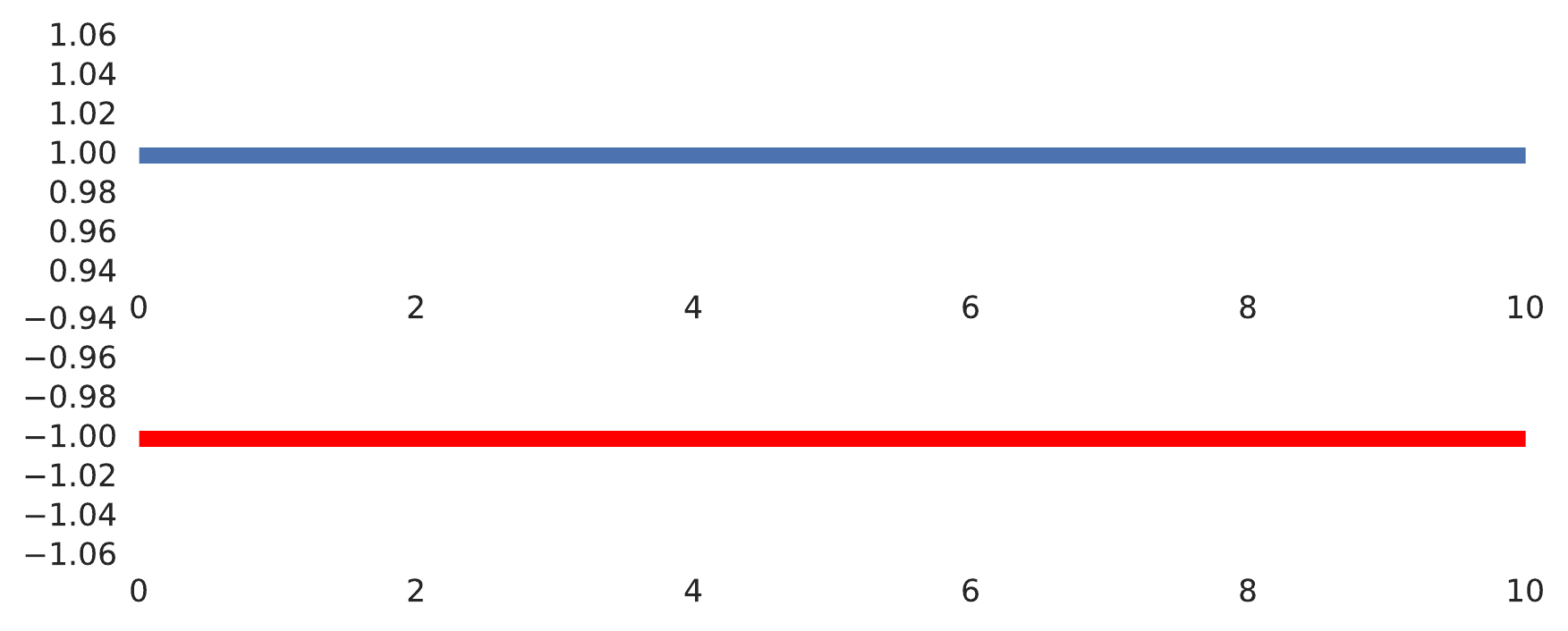}
}
\caption{Results on the foraging world using $m=2$ nutrients and $3$ types of food items: $\y_1 = (1,0)$, $\y_2 = (0,1)$, and $\y_3 = (1,1)$. Shaded regions are one standard deviation over $10$ runs. \label{fig:add_results_foraging_world12}}
\end{figure*}

The simplest scenario one could consider is one where the desirability functions associated with each nutrients are constant. In particular we will look at the scenario where both desirability functions are positive (Figure \ref{fig:add_foraging_world0}). In this case we benefit from picking up any of the two nutrients and the most desirable item is item $3$ which gives the agent a unit of each nutrient. As our keyboard was trained for cumulants corresponding to $\mathcal{W}_0 = \{(1,0),(0,1)\}$, the trained primitive options would be particularly suited for this task, as they are tuned to picking up the food items present in the environment. Performance of the player and comparison with a Q-learning agent are reported in Figure \ref{fig:add_results_foraging_world0}. The first thing to note is that our player can make effective use of the OK's options and converges very fast to a very good performance. Nevertheless, the Q-learning agent will eventually catch-up and could possibly surpass our players' performance, as our policy for the ``true'' cumulant induced by $w^*=(1,1)$ is possibly suboptimal. But this will require a lot more learning.

The second simple scenario we looked at is the one where one nutrient is desirable---$d_1(x)>0, \forall x$---and the other one is not: $d_2(x)<0, \forall x$ (Figure \ref{fig:add_foraging_world1}). In this case only one of the trained options will be useful, the one going for the nutrient that has a positive weight. But even this is one will be suboptimal as it will pick up equally food item $1$ ($\y_1=(1,0)$) and $3$ ($\y_3=(1,1)$) although the latter will produce no reward for the player. Moreover, sticking to this first option, which is the only sensible policy available to QO, the player cannot avoid items of type $2$ if they happen to be on their path to collecting one of the desirable items ($1$ or $3$). This accounts for the suboptimal level that this player, based solely on the trained options, achieves in Figure \ref{fig:add_results_foraging_world1}. 
We can also see that the only combined option that improved the performance of the player is $\w_3 =(1,-1)$ which corresponds exactly to the underlying reward scheme present in the environment at this time (this is akin to the scenario discussed in Section \ref{sec:combining_options}, $e_1-e_2$, where the agent is encouraged to walk ($e_1$) while avoiding grasping objects ($e_2$)). By adding this \textit{synthesised option} to the collection of primitive options, we can see that the player Q(3)-3 achieves a considerably better asymptotic performance and maintains a speedy convergence compared to our baseline (QL). It is also worth noting that, in the absence of any information about the dynamics of the domain, we can opt for a range of diverse combinations, like exemplified by QP(8), and let the player decide which of them is useful. This will mean learning with a larger set of options, which will delay convergence. At the same time this version of the algorithm manages to deal with both of these situations, and many more, as shown in Figures \ref{fig:add_results_foraging_world34} and \ref{fig:paper_results_foraging_world}, without any modification, being agnostic to the type of change in reward the player will encounter. We hypothesise that this is representative of the most common scenario in practice, and this is why in the main paper we focus our analysis on this scenario alone. Nevertheless, in this in-depth analysis, we aim to understand what different combination of the same set of options would produce and in which scenarios a player would be able to take advantage of these induced behaviours.

\begin{figure*}
\centering
\newcommand{\scl}{0.362}
\subfigure[Scenario $A3$  \label{fig:add_results_foraging_world3}]{
\includegraphics[scale=\scl]{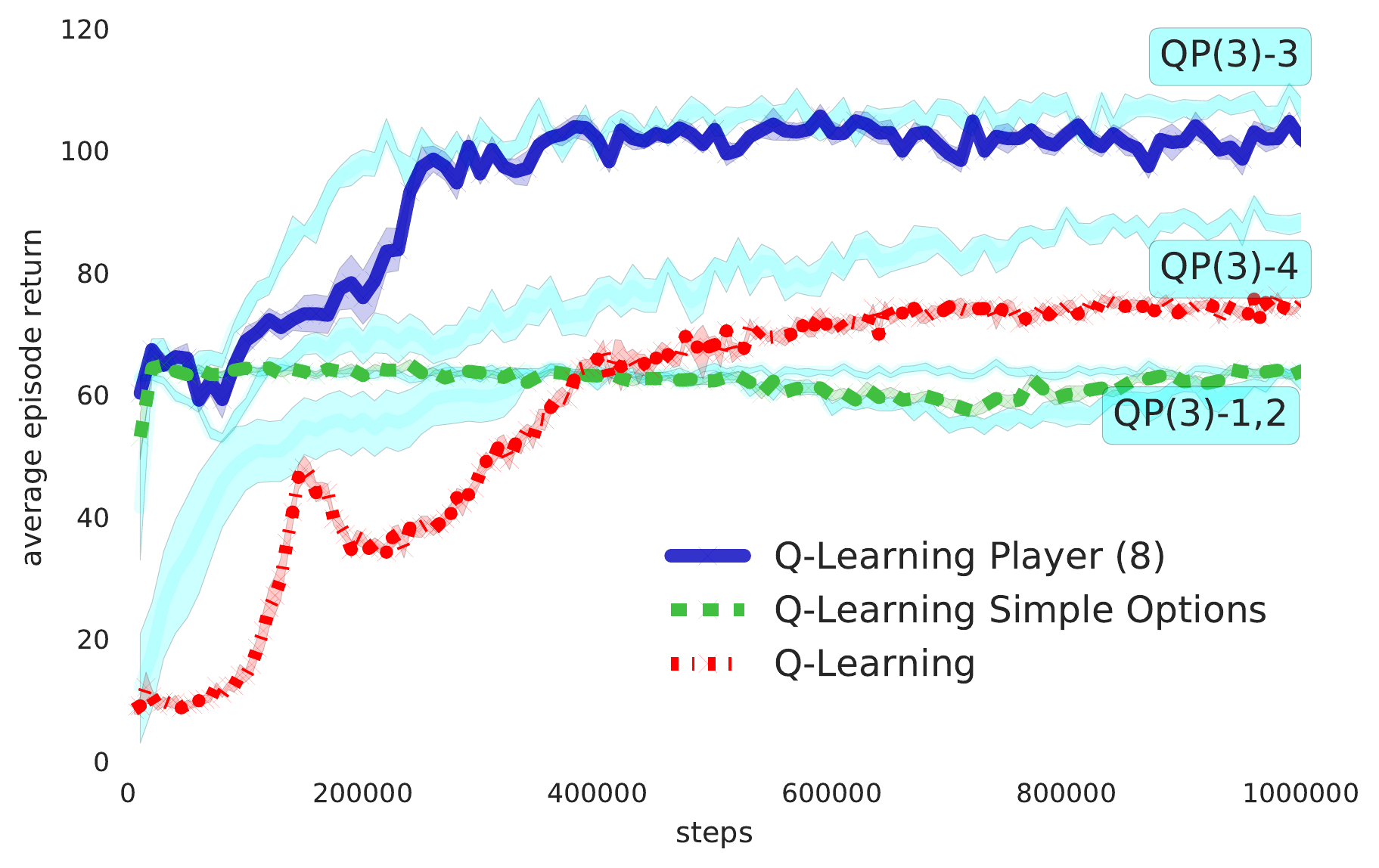}
} \hfill
\subfigure[Scenario $A4$  \label{fig:add_results_foraging_world4}]{
\includegraphics[scale=\scl]{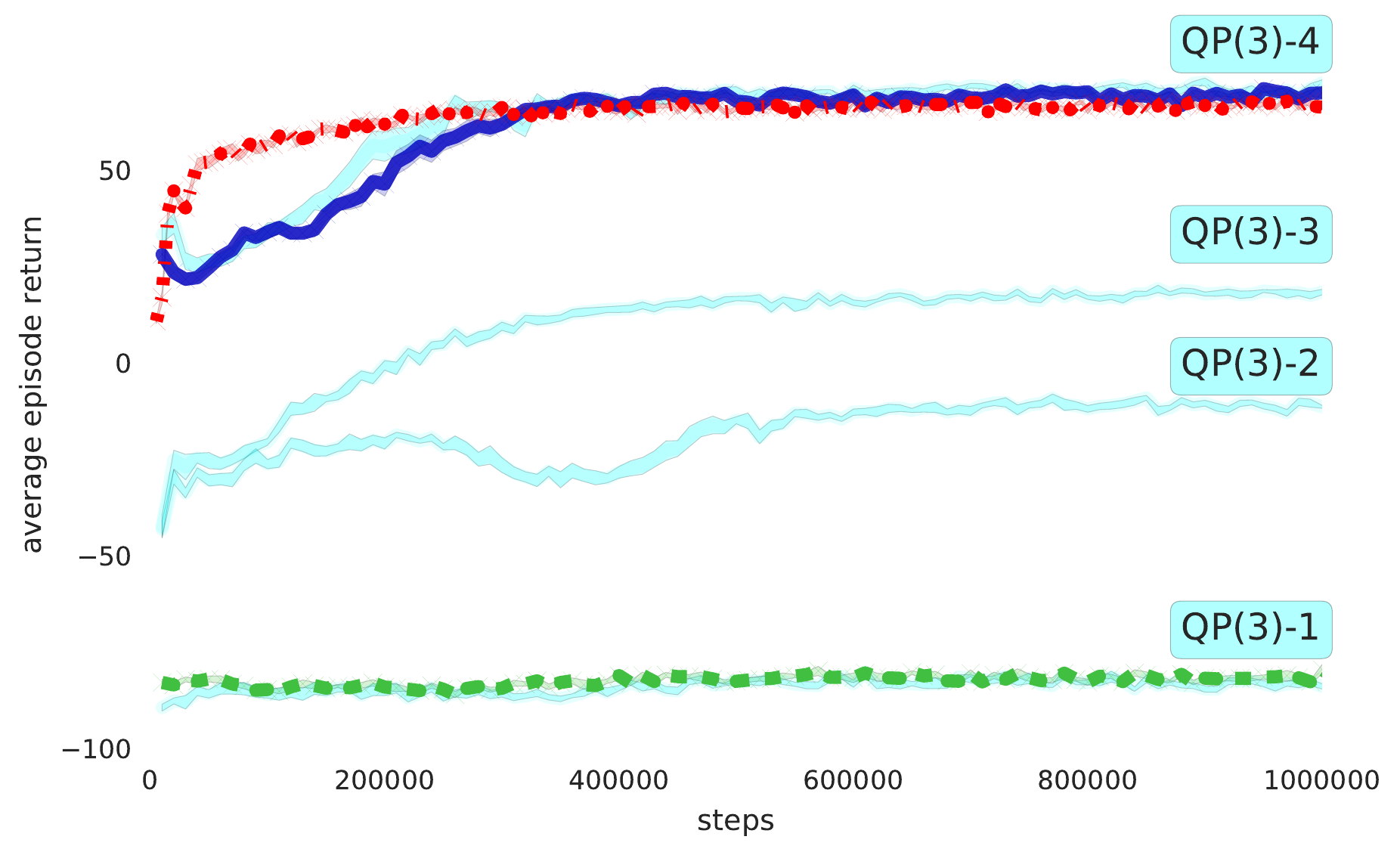}
}

\subfigure[Scenario $A3$: $d_i(x_i)$  \label{fig:add_foraging_world3}]{
\includegraphics[width=0.44\textwidth]{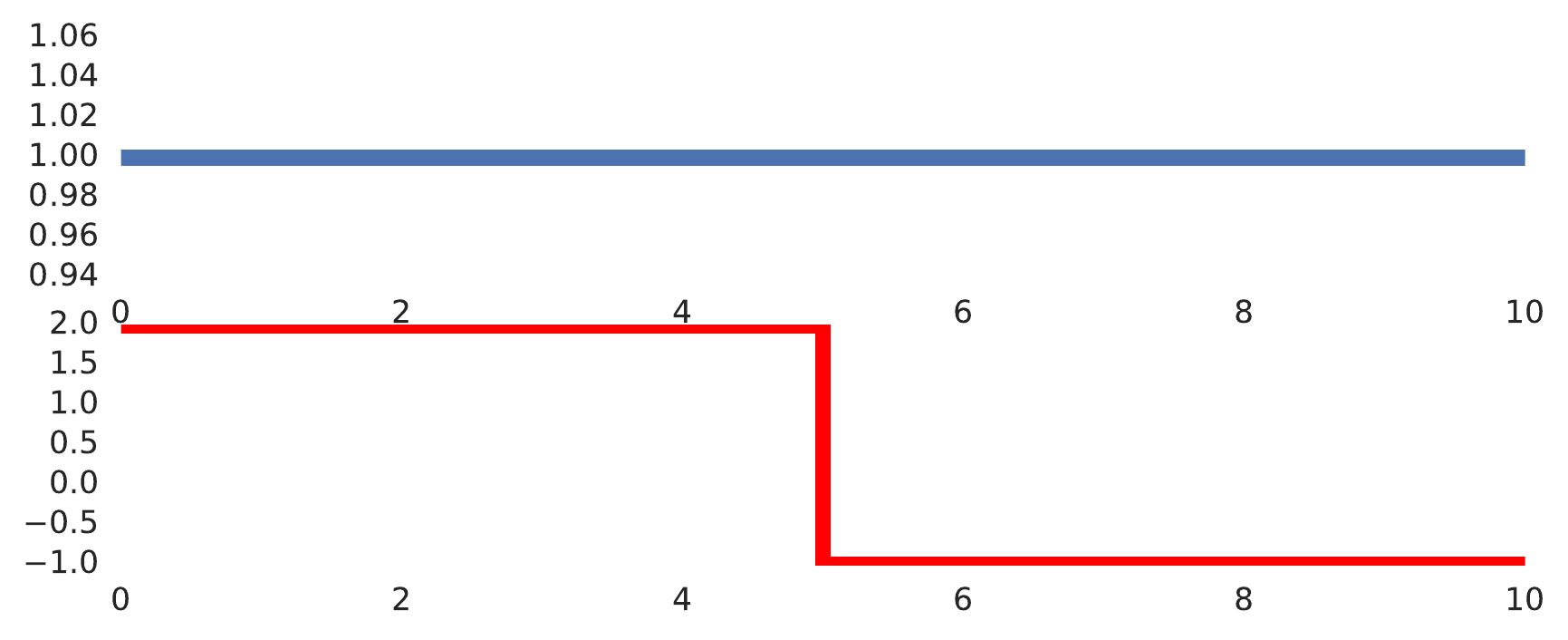}
} \hfill
\subfigure[Scenario $A4$: $d_i(x_i)$  \label{fig:add_foraging_world4}]{
\includegraphics[width=0.44\textwidth]{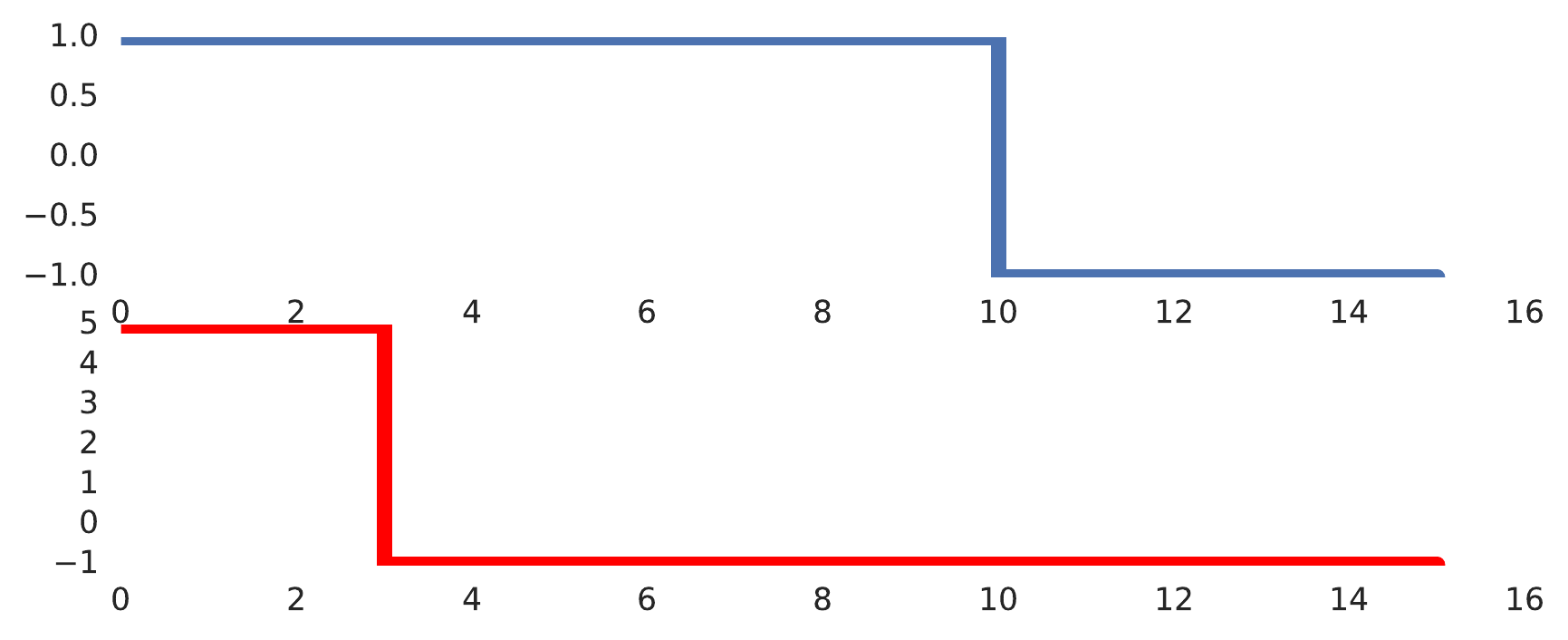}
}
\caption{Results on the foraging world using $m=2$ nutrients and $3$ types of food items: $\y_1 = (1,0)$, $\y_2 = (0,1)$, and $\y_3 = (1,1)$. Shaded regions are one standard deviation over $10$ runs. \label{fig:add_results_foraging_world34}}
\end{figure*}

Next we are going to look at a slightly more interesting scenario, where the desirability function changes over time as a function of the player's inventory. An intuitive scenario is captured in $A3$ (Figure \ref{fig:add_foraging_world3}) where for the second nutrient we will be considering a function that gives us positive reward until the number of units of this nutrient reaches a critical level---in this particular scenario $5$---, when the reward associated with it changes sign. The first nutrient remains constant with a positive value. We can think of the second nutrient here as something akin to certain types of food, like ''sugar'': at some point, if you have too much, it becomes undesirable to consume more. And thus you would have to wait until the leakage rate $l_i$ pushes this nutrient into a positive range before attempting to pick it up again. Conceptually this is a combination of the scenarios we have explored previously in Figure \ref{fig:add_results_foraging_world12}, but now the \textit{two situations would occur in the same episode}. Results are presented in Figure \ref{fig:add_results_foraging_world3}. As before, we can see that adding the synthesised option corresponding to $\w_3 = (1,-1)$, emulating the reward structure in the second part of the episode, gives us a considerable boost in performance as compared to the primitive set $\mathcal{W}_0$ (QO). Moreover, we can see again that the player converges considerably faster than the Q-learning agent which now encounters a switch in the reward structure based on inventory. This change in the desirability function makes this scenario considerably more challenging for the flat agent, while the player has the right level of abstraction to deal with this problem effectively.

The fourth scenario considered in this section is a small modification of the one above, where both nutrients have the ``sugar'' profile: they both start positive and at one point become negative---see Figure \ref{fig:add_foraging_world4}. We consider different thresholds at which this switch happens for each nutrient, to show that we can deal with asymmetries not present in pre-training. The results are shown in Figure \ref{fig:add_results_foraging_world4}.
Now we can see that this small modification leads to a very different learning profile. The first thing to notice is that the player based on only primitive options, QO, does very poorly in this scenario. This is because this player can only act based on options that pick up items and due to the length of our episodes ($300$) this player will find itself most of the time in the negative range of these desirability functions. Moreover the player will be unable to escape this range as it will continue to pick up items, resulting in more nutrients being added to its inventory, since these are the only options available to it. On the other hand we can see that by considering combinations of these options our players can do much better. In particular, given the above desirability functions, we expect negative combinations to be helpful. And, indeed, when we add $\w_2$, $\w_3$ or $\w_4$ to the set of primitive options, we can see that the resulting players QP(3)-2,3,4 perform considerably better than QO.   
Unsurprisingly, adding a positive-only combination, like $\w_1$, does not help performance, as even in the positive range this option would be suboptimal and will mimic the performance of the primitive set (as already seen in scenario A1, Figure \ref{fig:add_results_foraging_world0}). 
It is worth noting that in this case we are on par with QL, but keep in mind that this scenario was chosen a bit adversarially against our OK players. Remember this is a scenario where planning on top of the trained options alone would lead to a very poor performance. Nevertheless we have seen that by considering combinations of options, our OK players can achieve a substantially better performance. This is genuinely remarkable and illustrate the power of this method in combining options in cumulant space: \textit{even if the primitive options do not elicit a good plan, combined options synthesised on top of these primitive options can lead to useful behaviour and near optimal performance}.

\begin{figure*}
\centering
\newcommand{\scl}{0.4}
\subfigure[Scenario $1$  ]{
\includegraphics[scale=\scl]{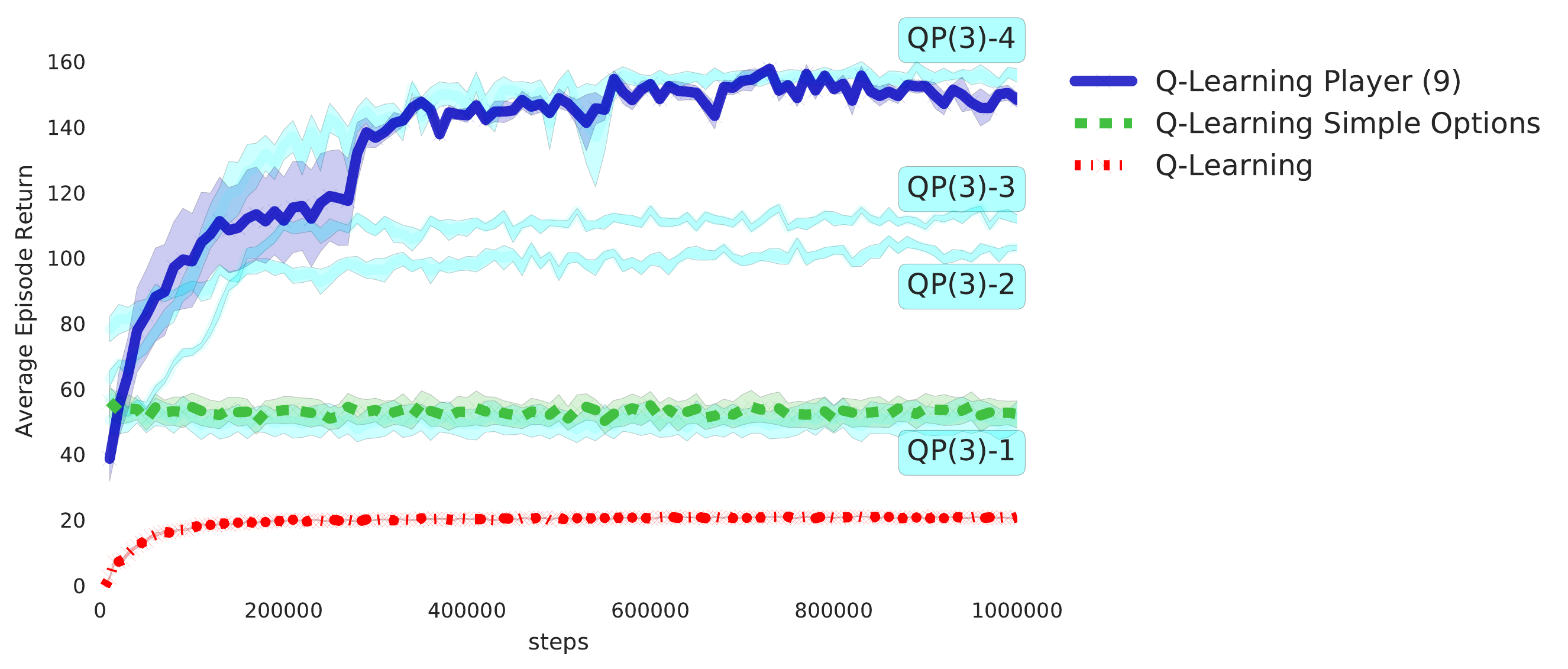}
} \hfill
\subfigure[Scenario $2$]{
\includegraphics[scale=\scl]{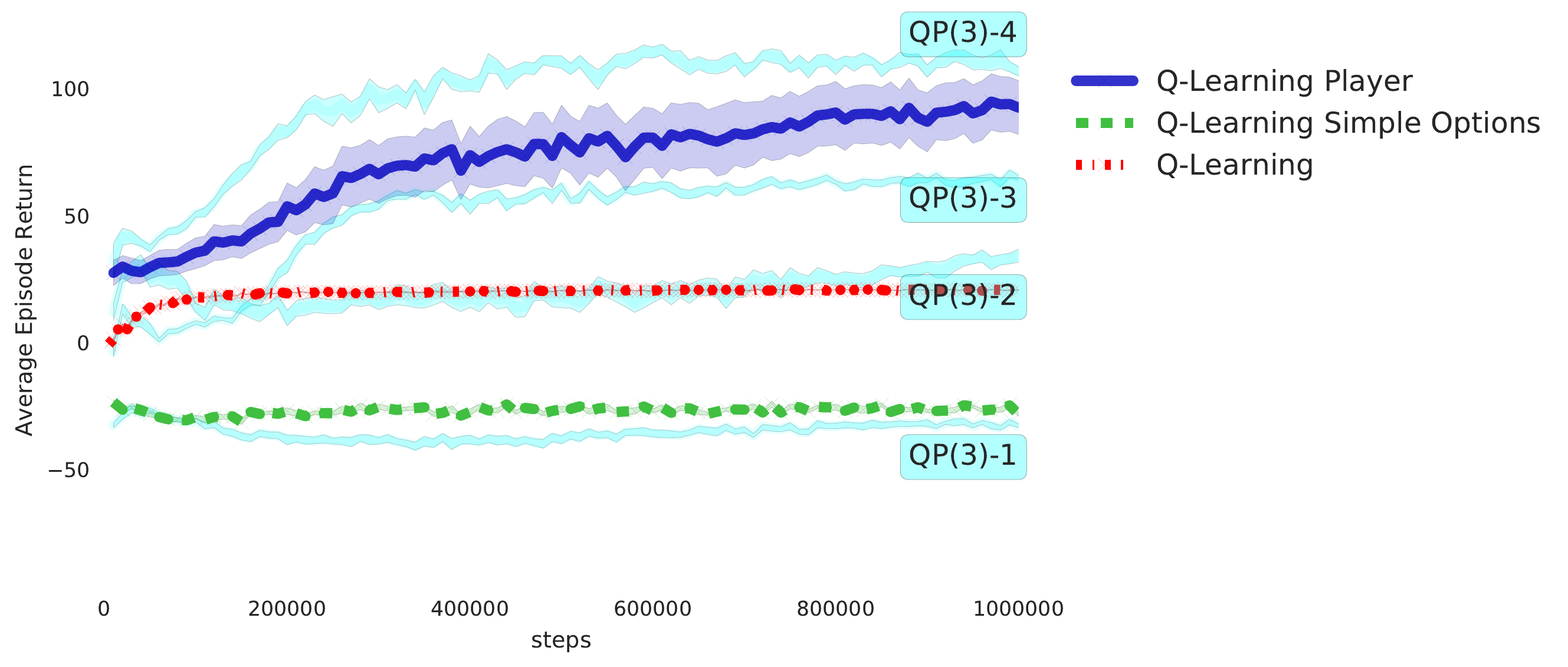}
}

\subfigure[Scenario $1$: $d_i(x_i)$]{
\includegraphics[width=0.44\textwidth]{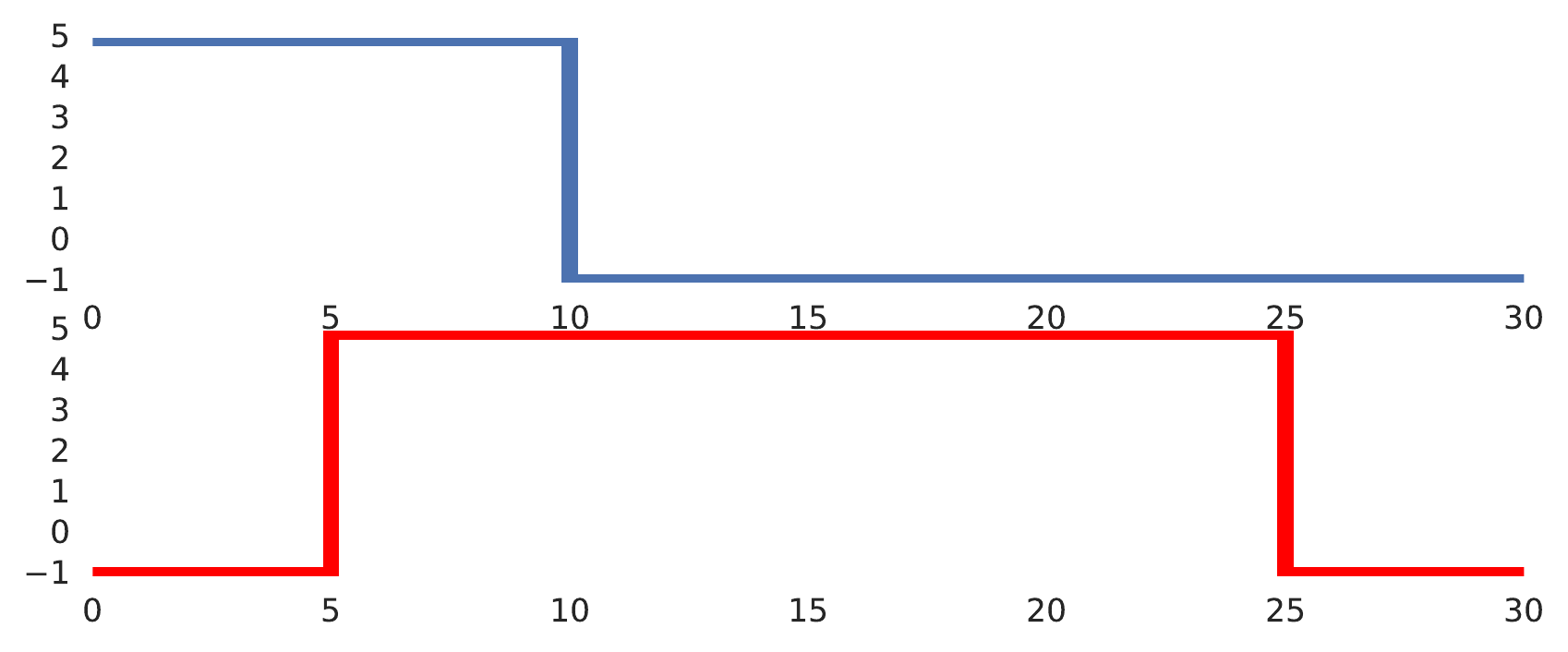}
} \hfill
\subfigure[Scenario $2$: $d_i(x_i)$]{
\includegraphics[width=0.44\textwidth]{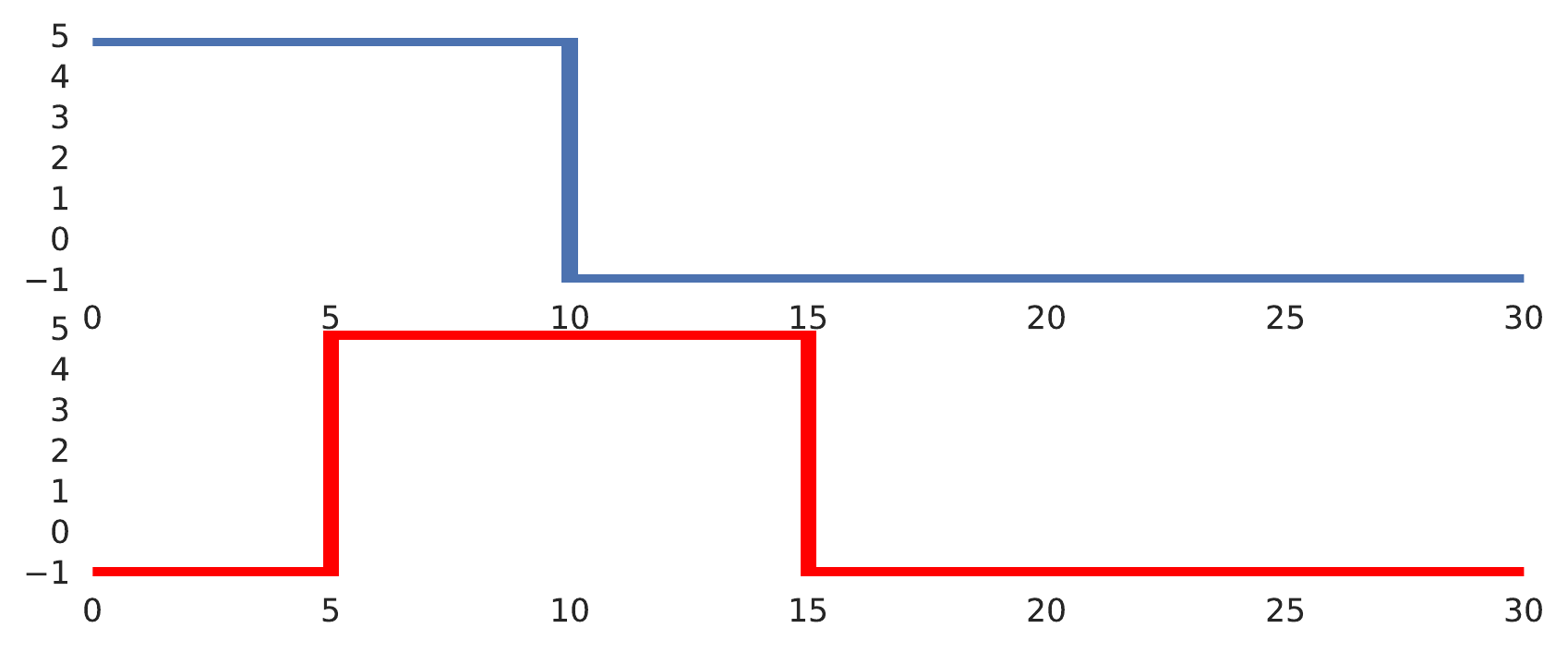}
}
\caption{Results on the foraging world using $m=2$ nutrients and $3$ types of items: $\y_1 = (1,0)$, $\y_2 = (0,1)$, and $\y_3 = (1,1)$. Shaded regions are one standard deviation over $10$ runs. \label{fig:paper_results_foraging_world}}
\end{figure*}

Lastly, for convenience we included the two scenarios presented in the main paper, as well as their desirability functions, in Figure \ref{fig:paper_results_foraging_world}. As in previous analysis, we include the performance of the QP(3)-$i$ players to illustrate how each of these combined options influences the performance. It is worth noting that in all of these scenarios, including the ones in the main text, the same keyboard $\Qec$ was (re-)used and that player QP(8), which considers a diverse set of the proposed combinations, can generally recover the best performance. This is an example of an agnostic player that learns to use the most useful combined options in its set, depending on the dynamics it observes in the environment. Moreover, in all of these scenarios we can clearly outperform or at least match the performance of the flat agent, QL, and the agent that only uses basic options, QO. This shows that the proposed hierarchical agent can effectively deal with structural changes in the reward function, by making use of the combined behaviour produced by GPE and GPI (Section \ref{sec:gpe_gpi}). 

\subsection{Moving-target arena}

Figure~\ref{fig:results_moving_target_arena_lenghts} shows additional OK's results on the moving-target arena as we change the length of the options. We vary the options' lengths by changing the value of $k$ in the definition of the cumulants~(\ref{eq:ecm_k_steps_policy}).
As a reference, we also show the performance of flat DPG in the original continuous action space \A. 

\begin{figure}
\centering
\subfigure[$Q$-learning players using $|\W|=8$ combined options]{
\includegraphics[scale=0.45]{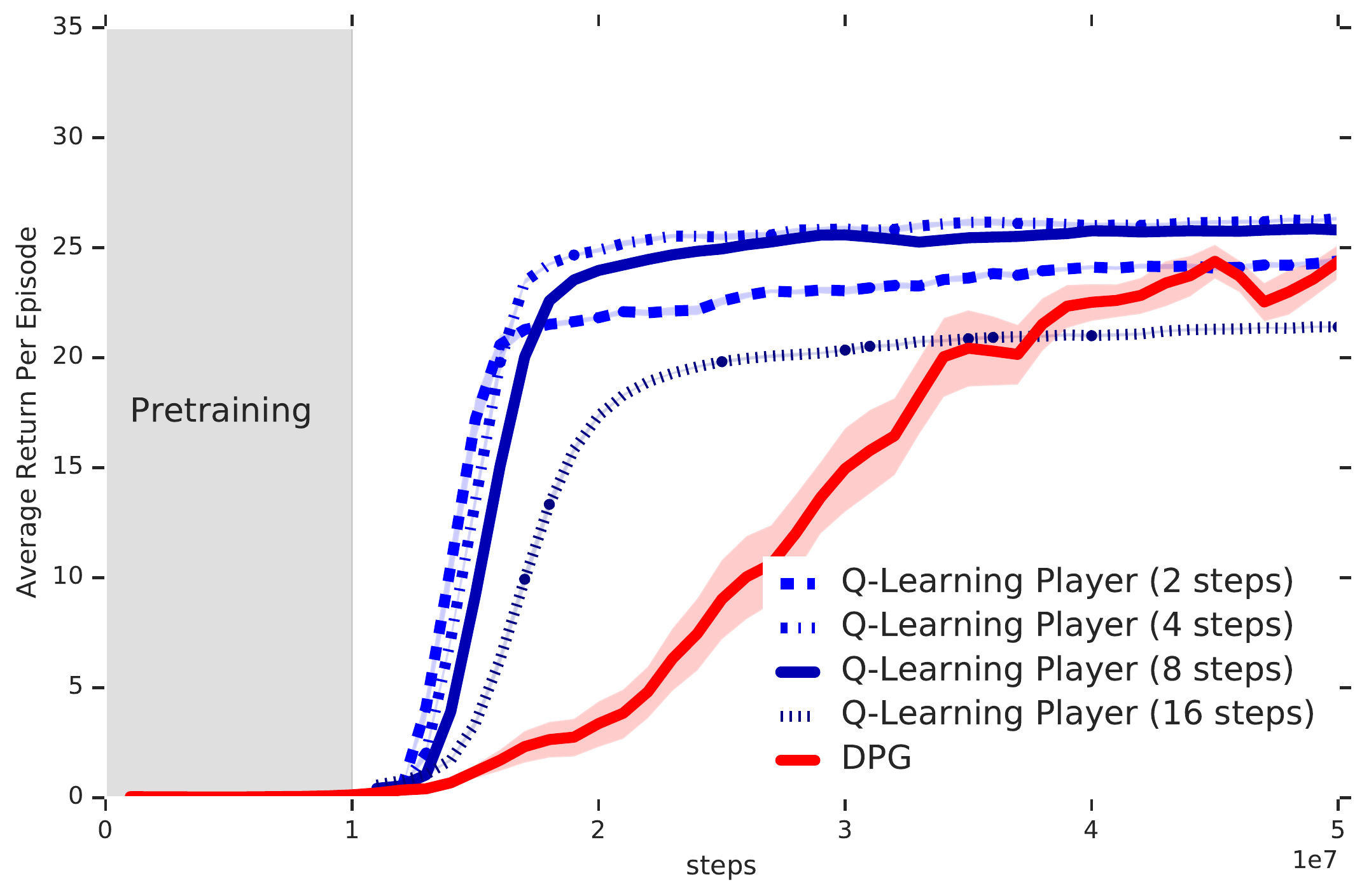}
} 
\subfigure[DPG player]{
\includegraphics[scale=0.45]{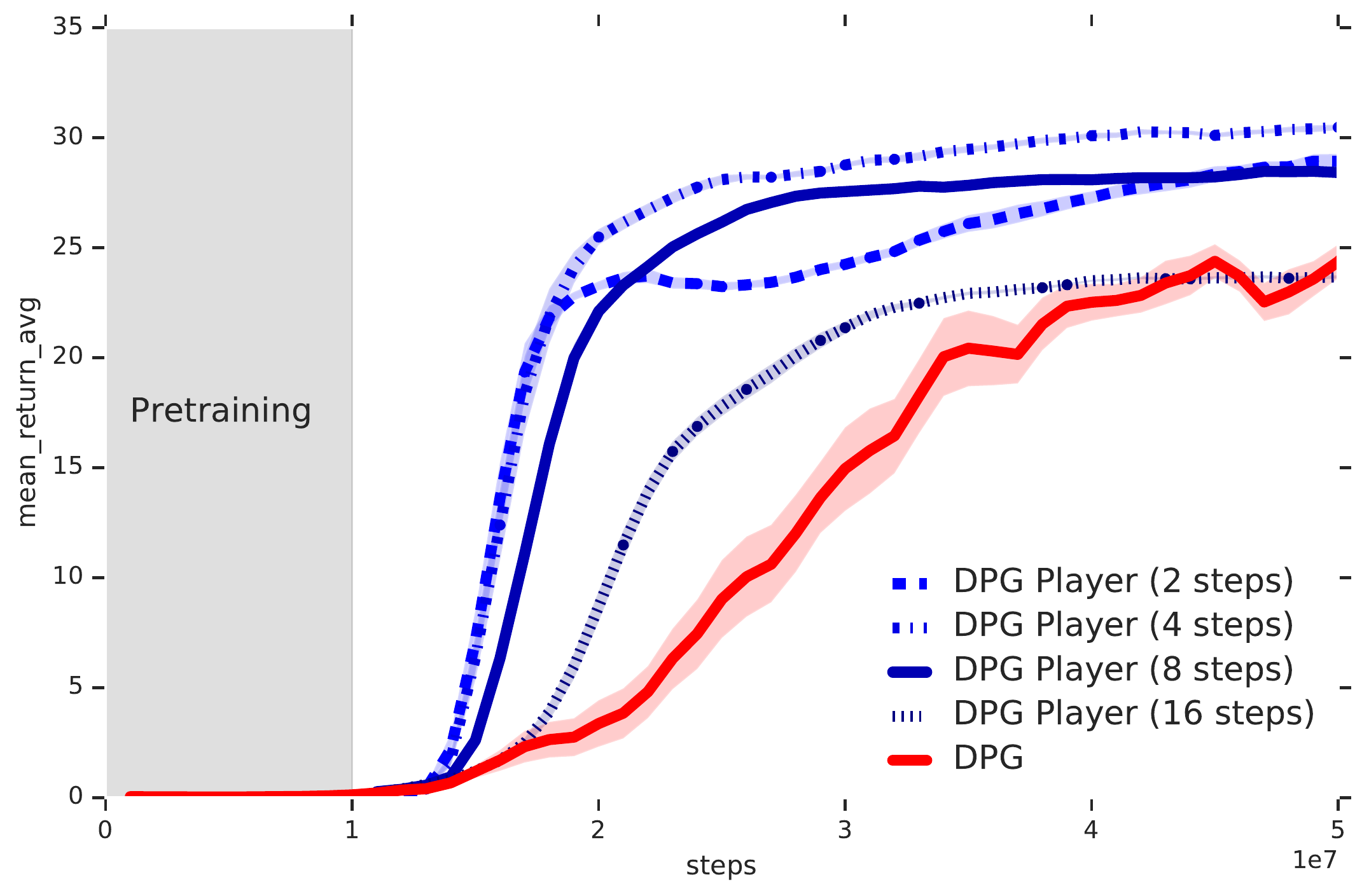}
} 
\caption{ Results on the moving-target arena for options of different lengths.  The number of steps corresponds to the value of $k$ in~(\ref{eq:ecm_k_steps_policy}). Shaded regions are one standard deviation over $10$ runs.  \label{fig:results_moving_target_arena_lenghts}}
\end{figure}

\makeatletter
\apptocmd{\thebibliography}{\global\c@NAT@ctr 34\relax}{}{}
\makeatother

\end{document}

%% file: defs.tex
\newcommand{\cp}{\citep}
\newcommand{\ca}{\citeauthor}

\newcommand{\ct}{\citet}
\newcommand{\ctp}[1]{\ca{#1}'s~\cp{#1}}

\newcommand{\mat}[1]{\ensuremath{\boldsymbol{{#1}}}}

\newcommand{\defi}{\ensuremath{\equiv}}
\newcommand{\cm}{\ensuremath{c}}
\newcommand{\Cm}{\ensuremath{C}}
\newcommand{\C}{\ensuremath{\mathcal{C}}}
\renewcommand{\H}{\ensuremath{\mathcal{H}}}

\newcommand{\ecm}{\ensuremath{e}}
\newcommand{\api}{\ensuremath{\omega}}

\newcommand{\Ec}{\ensuremath{\mathcal{E}}}
\newcommand{\curr}{\ensuremath{\mathrm{last}}}
\newcommand{\length}{\ensuremath{\mathrm{length}}}

\newtheorem{proposition}{Proposition}

\newcommand{\R}{\ensuremath{\mathbb{R}}}

\newcommand{\E}{\ensuremath{\mathbb{E}}}

\newcommand{\w}{\mat{w}}

\renewcommand{\t}{\ensuremath{\top}}
\newcommand{\y}{\mat{y}}

\newcommand{\argmax}{\ensuremath{\mathrm{argmax}}}
\renewcommand{\min}{\ensuremath{\mathrm{min}}}
\newcommand{\nulls}{\ensuremath{s_{\emptyset}}}

\newcommand{\Qt}{\ensuremath{\tilde{Q}}}
\newcommand{\Qec}{\ensuremath{\mathcal{Q}_{\Ec}}}

\newcommand{\la}{\ensuremath{\leftarrow}}

\renewcommand{\S}{\ensuremath{\mathcal{S}}}

\newcommand{\A}{\ensuremath{\mathcal{A}}}
\newcommand{\W}{\ensuremath{\mathcal{W}}}
\newcommand{\I}{\ensuremath{\mathcal{I}}}

\newcommand{\params}{\ensuremath{\mat{\theta}}}

\newcommand{\ind}{\ensuremath{\mat{1}}}

\newcommand{\vv}{\ensuremath{\mat{v}}}
\newcommand{\vx}{\ensuremath{\mat{x}}}
\newcommand{\vy}{\ensuremath{\mat{y}}}
\newcommand{\vtheta}{\mat{\theta}}